\documentclass{article}

\usepackage{PRIMEarxiv}

\usepackage[utf8]{inputenc} 
\usepackage[T1]{fontenc}    
\usepackage{hyperref}       
\usepackage{url}            
\usepackage{booktabs}       
\usepackage{amsfonts}       
\usepackage{nicefrac}       
\usepackage{microtype}      
\usepackage{lipsum}
\usepackage{fancyhdr}       
\usepackage{graphicx}       
\graphicspath{{media/}}     
\usepackage[acronym,toc]{glossaries}
\usepackage{glossary-mcols}

\usepackage{amssymb}
 \usepackage{amsthm}
 \usepackage{amsmath}
\usepackage{tikz}
\usepackage{amsopn}
\usepackage[all]{xy}
\usepackage{enumitem}
\usepackage{algpseudocode,algorithm}

\newcommand{\MB}{\mathbb{MB}}
\newcommand{\CB}{\mathbb{CB}}
\newcommand{\EB}{\mathbb{EB}}
\newcommand{\HC}{\mathbb{HC}}
\newcommand{\AF}{\mathcal{F}}
\newcommand{\A}{\mathcal{A}}

\newcommand{\D}{\mathcal{D}}

\newcommand{\Att}{\mathtt{Att}}

\newcommand{\supp}{\mathrm{supp}}

\newcommand{\fix}{\mathrm{fix}}

\def\F{\mathcal{F}}
\def\G{\mathcal{G}}

\def\Base{\OP{Base}}
\def\wBase{\OP{w-Base}}
\def\BpHI{\mathcal{B^+HI}}
\def\diag{\operatorname{diag}}
\def\End{\OP{End}}
\def\Func{\OP{Func}}
\def\inv{\operatorname{inv}}
\def\Mat{\operatorname{Mat}}
\def\Scheme{\operatorname{Scheme}}
\def\RR{\mathbb{R}}

\def\SD{\mathcal{D}{yn}}

\newcommand{\xto}[1]{\xrightarrow{#1}}
\newcommand{\OP}[1]{\operatorname{#1}}

\newtheorem*{prop*}{Proposition}

\newtheorem{proposition}{Proposition}[section]
\newtheorem{theorem}[proposition]{Theorem}

\newtheorem{lemma}[proposition]{Lemma}

\theoremstyle{definition}
\newtheorem{example}[proposition]{Example}
\newtheorem{definition}[proposition]{Definition}
\newtheorem{defn}[proposition]{Definition}

\newtheorem{problem}[proposition]{Problem}
\newtheorem{remark}[proposition]{Remark}
\newtheorem*{remarkx}{Remark}
\newtheorem*{scholiumx}{Scholium}

\title{Abstract Weighted Based Gradual Semantics in Argumentation Theory}

\author{
  Assaf Libman, Nir Oren \\
  University of Aberdeen \\
  Scotland \\
  \texttt{\{a.libman, n.oren\}@abdn.ac.uk} \\
   \And
  Bruno Yun \\
  Univ Lyon, UCBL, CNRS, INSA Lyon, LIRIS\\
  UMR5205, F-69622 Villeurbanne \\
  France\\
  \texttt{bruno.yun@univ-lyon1.fr} \\
}

\begin{document}

\newglossaryentry{A}{
name=$\A$,
description={finite set of arguments}
}

\newglossaryentry{AF}{
name=$\AF$,
description={weighted AF}
}

\newglossaryentry{D}{
name=$\D$,
description={binary attack relation}
}

\newglossaryentry{w}{
name=$w$,
description={weighting function}
}

\newglossaryentry{G}{
name=$\G$,
description={underlying AF}
}

\newglossaryentry{Att}{
name=$\Att$,
description={attackers of an argument}
}

\newglossaryentry{Sigma}{
name=$\Sigma$,
description={gradual semantics}
}

\newglossaryentry{SigmaAF}{
name=$\Sigma^\AF$,
description={scoring function associated to $\AF$}
}

\newglossaryentry{SigmaMB}{
name=$\Sigma_{\MB}$,
description={weighted max-based grad. sem.}
}

\newglossaryentry{OPMBk}{
name=$\OP{MB}_k$,
description={recursively defined MB func.}
}

\newglossaryentry{SigmaCB}{
name=$\Sigma_{\CB}$,
description={weighted card-based grad. sem.}
}

\newglossaryentry{OPCBk}{
name=$\OP{CB}_k$,
description={recursively defined CB func.}
}

\newglossaryentry{Atts}{
name=$\Att^*$,
description={attackers with non-zero weight.}
}

\newglossaryentry{SigmaHC}{
name=$\Sigma_{\HC}$,
description={weighted h-categorizer grad. sem.}
}

\newglossaryentry{OPHCk}{
name=$\OP{HC}_k$,
description={recursively defined HC func.}
}

\newglossaryentry{trianglerighteq}{
name=$\trianglerighteq$,
description={at least as preferred as}
}

\newglossaryentry{simeq}{
name=$\simeq$,
description={equally preferred}
}

\newglossaryentry{triangleright}{
name=$\triangleright$,
description={strictly preferred to}
}

\newglossaryentry{trianglelefteq}{
name=$\trianglelefteq$,
description={at most as preferred as}
}

\newglossaryentry{FuncAB}{
name={$\Func(A,B)$},
description={functions from $A$ to $B$}
}

\newglossaryentry{EndA}{
name={$\End(A)$},
description={functions from $A$ to $A$}
}

\newglossaryentry{Tseq}{
name={T-sequence},
description={}
}

\newglossaryentry{X}{
name={$X$},
description={space of scores}
}

\newglossaryentry{W}{
name={$W$},
description={space of weights}
}

\newglossaryentry{K}{
name={$K$},
description={space of vertices}
}

\newglossaryentry{sigma}{
name={$\sigma$},
description={scoring scheme}
}

\newglossaryentry{Scheme}{
name={$\Scheme$},
description={space of all scoring schemes}
}

\newglossaryentry{Dsigma}{
name={$D_\sigma$},
description={space of acceptability degree of $\sigma$}
}

\newglossaryentry{MatnnnR}{
name={$\Mat_{n \times n}(\RR)$},
description={$n \times n$ matrices whose elements are  in $\RR$}
}

\newglossaryentry{MatnnnR+}{
name={$\Mat_{n \times n}^+ (\RR)$},
description={$n \times n$ matrices whose elements are  in $\RR^+$}
}

\newglossaryentry{awgs}{
name={abstract weighted gradual semantics},
description={}
}

\newglossaryentry{inv}{
name={$\inv$},
description={inverse function}
}

\newglossaryentry{invP}{
name={inverse problem},
description={}
}

\newglossaryentry{reflP}{
name={reflection problem},
description={}
}

\newglossaryentry{toporflP}{
name={topological reflection problem},
description={}
}

\newglossaryentry{homeomorphism}{
name={homeomorphism},
description={}
}

\newglossaryentry{ppop}{
name={projective preference ordering problem},
description={}
}

\newglossaryentry{radialP}{
name={radiality problem},
description={}
}

\newglossaryentry{increasing}{
name={increasing},
description={}
}

\newglossaryentry{homogeneous}{
name={homogeneous},
description={}
}

\newglossaryentry{non-negative}{
name={non-negative},
description={}
}

\newglossaryentry{bounded}{
name={bounded},
description={}
}

\newglossaryentry{BpHIn}{
name={$\BpHI_n(X)$},
description={bounded, non-negative, homogeneous and increasing functions from $X$ to $\RR^n$}
}

\newglossaryentry{BpHI}{
name={$\BpHI(X)$},
description={bounded, non-negative, homogeneous and increasing functions from $X$ to $\RR$}
}

\newglossaryentry{scoringbase}{
name={scoring base},
description={}
}

\newglossaryentry{BaseX}{
name={$\Base(X)$},
description={set of all scoring bases}
}

\newglossaryentry{Tcf}{
name={$T_{(c,f)}$},
description={a particular endomorphism from the scoring base $(c,f)$}
}

\newglossaryentry{scoringdynamics}{
name={scoring dynamics},
description={}
}

\newglossaryentry{SDX}{
name={$\SD(X)$},
description={collection of all scoring dynamics}
}

\newglossaryentry{fixT}{
name={$\fix(T)$},
description={fix point of a scoring dynamics $T$}
}

\newglossaryentry{orderpreserving}{
name={order preserving},
description={}
}

\newglossaryentry{orderreversing}{
name={order reversing},
description={}
}

\newglossaryentry{wsb}{
name={weighted scoring base},
description={}
}

\newglossaryentry{kv}{
name={$(\kappa,\varphi)$},
description={a weighted scoring base}
}

\newglossaryentry{kappa}{
name={$\kappa$},
description={first elt. of a weighted scoring base}
}

\newglossaryentry{varphi}{
name={$\varphi$},
description={second elt. of a weighted scoring base}
}

\newglossaryentry{wBaseX}{
name={$\wBase(X)$},
description={set of all weighted scoring bases}
}

\newglossaryentry{sigmab}{
name={$\sigma_{b}$},
description={scoring scheme associated to the weighted scoring base $b$}
}

\newglossaryentry{awbgs}{
name={abstract weighted based gradual semantics},
description={}
}

\newglossaryentry{supp}{
name={$\supp(x)$},
description={the support of $x\in \RR^n$}
}

\newglossaryentry{pressupport}{
name={preserves supports},
description={}
}

\newglossaryentry{independentsupports}{
name={independent of supports},
description={}
}

\newglossaryentry{fI}{
name={$f^I$},
description={the value of $f$ on all $x \in X$ with support $I$}
}

\newglossaryentry{dri}{
name={discerning right inverse},
description={}
}

\newglossaryentry{preceqclosed}{
name={$\preceq$-closed},
description={closure of a set under $\preceq$}
}

\newglossaryentry{supppreceqclosed}{
name={$(\supp,\preceq)$-closed},
description={closure of a set under $\preceq$ and $\supp$}
}

\newglossaryentry{lpnom}{
name={$||(x_1,\dots,x_n)||_p$},
description={$L^p$-norm on $\RR^n$}
}

\newglossaryentry{xwedgey}{
name={$x \wedge y$},
description={element-wise product of vectors}
}

\newglossaryentry{ai*}{
name={$a_{i,*}$},
description={the $i^\text{th}$ row of $A$}
}

\newglossaryentry{LplambdamuA}{
name={$\sigma_{(L^p,\lambda,\mu,A)}$},
description={the $(L^p,\lambda,\mu,A)$-based scoring scheme}
}

\newglossaryentry{awlplambdamuAbgs}{
name={abstract weighted $(L^p,\lambda,\mu)$-based gradual semantics},
description={}
}

\newglossaryentry{SigmaEB}{
name={$\Sigma_{\EB}$},
description={weighted Euclidean-based grad. sem.}
}

\newglossaryentry{invHC}{
name={\ensuremath{\inv^\G_{\HC}}},
description={discerning right inverses of $\Sigma_{\HC}$}
}

\newglossaryentry{invMB}{
name={\ensuremath{\inv^\G_{\MB}}},
description={discerning right inverses of $\Sigma_{\MB}$}
}

\newglossaryentry{invCB}{
name={\ensuremath{\inv^\G_{\CB}}},
description={discerning right inverses of $\Sigma_{\CB}$}
}

\newglossaryentry{OPEBk}{
name={\ensuremath{\OP{EB}_k}},
description={recursively defined EB func.}
}

\newglossaryentry{invEB}{
name={\ensuremath{\inv_{\EB}^\G}},
description={discerning right inverses of $\Sigma_{\EB}$}
}

\newglossaryentry{delta}{
name={\ensuremath{\delta}},
description={discount factor}
}

\newglossaryentry{gamma}{
name={\ensuremath{\gamma}},
description={a path}
}

\newglossaryentry{Att2i}{
name={\ensuremath{\Att_2(i)}},
description={all paths of length 2 ending at $i$}
}

\maketitle

\begin{abstract}
Weighted gradual semantics provide an acceptability degree to each argument representing the strength of the argument, computed based on factors including background evidence for the argument, and taking into account interactions between this argument and others.

We introduce four important problems linking gradual semantics and acceptability degrees. First, we reexamine the inverse problem, seeking to identify the argument weights of the argumentation framework which lead to a specific final acceptability degree. Second, we ask whether the function mapping between argument weights and acceptability degrees is injective or a homeomorphism onto its image. Third, we ask whether argument weights can be found when preferences, rather than acceptability degrees for arguments are considered. Fourth, we consider the topology of the space of valid acceptability degrees, asking whether ``gaps'' exist in this space.
While different gradual semantics have been proposed in the literature, in this paper, we identify a large family of  weighted gradual semantics, called \textit{abstract weighted based gradual semantics}. These generalise many of the existing semantics while maintaining desirable properties such as convergence to a unique fixed point.
We also show that a sub-family of the weighted gradual semantics, called \textit{abstract weighted $(L^p,\lambda,\mu)$-based gradual semantics} and which include well-known semantics, solve all four of the aforementioned problems.

\end{abstract}

\keywords{Argumentation \and Gradual semantics \and Inverse problems}

\section{Introduction}

In the context of Dung's abstract argumentation \cite{dung_acceptability_1995}, we consider a set of abstract arguments and a binary attack relation between them, encoding these as a directed graph. Argumentation semantics then identify which sets of arguments are justified together by considering inter-argument interactions \cite{baroni_introduction_2011,verheij1996two,DBLP:conf/comma/Caminada06,DBLP:journals/ai/DungMT07}. Within the argumentation community, there has been increasing interest in so-called \emph{ranking-based} semantics \cite{amgoud_ranking-based_2013,bonzon_comparative_2016,DBLP:journals/jancl/BonzonDKM23,yun20ranking}. These aim to identify a ranking over the arguments, with higher ranked arguments considered more justified 
than arguments ranked lower. 
 One approach to creating such a ranking, called \textit{gradual semantics} \cite{DBLP:conf/kr/AmgoudD18a,yun_2021_gradual,besnard_logic-based_2001}, involves associating a numerical \emph{acceptability degrees} to all arguments within the system, with the final ranking computed according to the numeric ordering. Furthermore, some ranking-based semantics compute the acceptability degree of an argument based not only on the topology of the argumentation framework, but also based on some \emph{initial weight} assigned to each argument. These \emph{weighted gradual semantics}, exemplified by the weighted \emph{max-based}, \emph{card-based} and \emph{h-categorizer} semantics, are widely studied and shown to have various desirable properties \cite{AMGOUD2022103607,oren2022inverse,DBLP:conf/kr/AmgoudD18a}.

While a specific weighted gradual semantics takes an argumentation system and initial weights as input, and outputs the arguments' acceptability degrees, the \emph{inverse problem} seeks to identify the initial weights for a given semantics, argumentation system and acceptability degrees \cite{oren2022inverse}. The latter paper used an iterative, approximate, numeric technique to identify such initial weights for the three weighted gradual semantics mentioned above. However, no guarantee was made that such a solution always exists. This paper moves away from an approximation-based approach and investigates the theory behind the inverse problem, providing important formal insights.
Namely, we ask ourselves the questions:
Can we always find a solution to the inverse problem for the three weighted gradual semantics?
If yes, can we generalise this result to a more general family of gradual semantics and how is this family characterised?
For any preference ordering on arguments (rather than simply numerical acceptability degrees), can we always find some weights on arguments so that the degree obtained will follow the preference ordering?
Our work makes the following contributions.

\begin{itemize}
\item
We establish a general family of ``abstract weighted based gradual semantics''. 
In doing so, we generalise the result of Pu et al. \cite{pu_argument_2014} to this large, general family of semantics. This result is used across the remainder of the paper.

\item We introduce a sub-family of abstract weighted based gradual semantics, called abstract weighted $(L^p,\lambda,\mu)$-based gradual semantics, which includes, among others, the three weighted gradual semantics mentioned above. 

\item We show that the inverse problem is solved for any semantics in this sub-family.

\item We show that any preference ordering between the arguments is realised by some acceptability degree in all the semantics in this sub-family.

\item We describe some topological properties of the acceptability degree space of semantics in this sub-family.

\end{itemize}

The importance of this sub-family of abstract weighted $(L^p,\lambda,\mu)$-based gradual semantics goes beyond giving a uniform treatment of the  three semantics above.
It allows us, among other things, to define new gradual semantics for which the above inverse-type problems are solved. In turn, inverse problems can be used --- for example --- to automatically elicit and identify users' preferences based on the acceptability degrees over arguments, and as part of an argument strategy to help decide which arguments to advance based on the implicit weights attached to arguments by other dialogue participants. Such applications are not however the focus of the current paper, and are left for future work. Instead, we focus on the formal underpinnings of inverse problems and gradual semantics.

While inverse-type problems for weighted gradual semantics are new, we also want to highlight that there has been some recent research in this area. Skiba et al. \cite{DBLP:conf/comma/SkibaTRHK22} have studied whether, given a ranking and a ranking-based semantics, we can find an argumentation framework (without weights) such that the selected ranking-based semantics induces the ranking when applied to the framework.
Kido et al., across several papers \cite{10.5555/3171642.3171679,DBLP:journals/jancl/KidoL22} have studied how attack relations can be determined  given sets of acceptable arguments. Their focus was on Dung style classical semantics. This line of work is similar to that of Oren and Yun \cite{DBLP:journals/corr/abs-2211-16118}, where the authors identify the complexity class for the problem of identifying a set of attacks that will yield the desired acceptability degrees for a set of weighted arguments with respect to three weighted gradual semantics.
Mailly \cite{DBLP:conf/clar/Mailly23} also studied realization problems for extension semantics but made used of auxiliary arguments, i.e., given a set of extensions $S$ and using $k$ auxiliary arguments, can we find $m$ argumentation frameworks such that the union of their extensions is exactly equal to $S$?
The existing body of research demonstrates the growing interest and diverse approaches in exploring various aspects of inverse problems in the argumentation domain.

The remainder of this paper is structured as follows. In Section \ref{sec:background}, we briefly recap the background notions and formalism needed in this paper. 
In Section \ref{subsec:scoring-schemes}, we introduce the notion of scoring schemes and rephrase weighted gradual semantics as collections of functions mapping adjacency matrices to scoring schemes. Then, we formalise four inverse-type problems in terms of scoring schemes, define the family of abstract weighted based gradual semantics, and show that they always converge to a unique fixed point.
In Section \ref{sec:solving-problems}, we provide conditions for based scoring schemes to solve the four inverse-type problems. This leads us to define the sub-family of abstract weighted $(L^p,\lambda,\mu)$-based gradual semantics which always solve the four problems in Section \ref{subsec:Lp lambda mu A}. 
In Section \ref{sec:new_semantics}, we show multiple examples of concrete semantics in this class and introduce new semantics.
We summarize and conclude our work in Section \ref{sec:conclusion}. The paper contains a significant amount of notation in order to provide general results. To aid the reader, we provide a table of notation in \ref{sec:symbols}.

\section{Background}
\label{sec:background}

This section introduces the necessary argumentation notions used in the rest of the paper.
Our departure point in this paper is the three weighted gradual semantics described in \cite{AMGOUD2022103607}. As input, these semantics take in a \emph{weighted argumentation framework} \cite{coste-marquis_selecting_2012,coste-marquis_weighted_2012,dunne_weighted_2011}.
A weighted argumentation framework is a Dung's abstract framework (with arguments and binary attacks) augmented with a weighting function which assigns a number between 0 and 1 to each argument.

\begin{definition}[WAF]

A \textbf{weighted argumentation framework} (WAF) is a triple \gls{AF} $= \langle \A, \D, w \rangle$ where \gls{A} is a finite set of arguments; \gls{D}$ \subseteq \A \times \A$ is a binary attack relation; and \gls{w}$:\A \to [0,1]$ is a \textbf{weighting function} assigning an initial weight to each argument.
The \textbf{underlying argumentation framework} (of $\AF$) is \gls{G}$ = \langle \A,\D\rangle$.
The set of attackers of an argument $a \in \A$ is denoted \gls{Att}$(a)=\{b \in \A | (b,a) \in \D\}$.

Upon ordering $\A=\{a_1,\dots,a_n\}$, the \textbf{adjacency matrix} of $\G$ is the $n\times n$-matrix $A$ with $A_{i,j}=1$ if $(a_j,a_i) \in \D$ and $A_{i,j}=0$ otherwise.
\end{definition}

A gradual semantics is a function that takes as input a weighted argumentation framework and produces a \textit{scoring function}. Note that the latter function is sometimes called a \textit{weighting} \cite{AMGOUD2022103607}, but we use our terminology to distinguish it from $w$.

\begin{definition}[Gradual Semantics] \label{def:gradual}
A \textbf{gradual semantics} \gls{Sigma} is a function that for each weighted argumentation framework $\AF= \langle \A, \D, w \rangle$, associates a \textbf{scoring function} \gls{SigmaAF}$:\A \to [0,1]$, whose values are referred to as the \textbf{acceptability degrees} of the arguments.
\end{definition}

Amgound, Doder and Vesic \cite{AMGOUD2022103607} investigated different desirable properties of gradual semantics (e.g., anonymity, directionality, independence, etc) for gradual semantics in weighted argumentation frameworks. They showed that three weighted gradual semantics satisfy many of these properties, and the argumentation literature thus focuses on these three semantics, which are specific instantiations of the gradual semantics of Definition \ref{def:gradual} above.



\begin{example}\label{example:MB semantics}
The \textbf{weighted max-based} gradual semantics \gls{SigmaMB}.
Given $\F=\langle \A,\D,w\rangle$, the acceptability degree of an argument $a\in \A$ is defined by
\[
\Sigma_{\MB}^\AF(a) = \lim_{k \to \infty} \OP{MB}_k(a)
\]
for the sequence of functions \gls{OPMBk}$ \colon \A \to [0,1]$ defined recursively by setting $\OP{MB}_0(a)=w(a)$ for all $a \in \A$ and 
\[
\OP{MB}_{k+1}(a) = \frac{w(a)}{1+\max\limits_{b \in \Att(a)} \OP{MB}_{k}(b)}
\]
\end{example}

\begin{example}\label{example:CB semantics}
The \textbf{weighted card-based} gradual semantics \gls{SigmaCB}.
Given $\F=\langle \A,\D,w\rangle$, the acceptability degree of an argument $a\in \A$ is
 \[
\Sigma^\AF_{\CB}(a) = \lim_{k \to \infty} \OP{CB}_k(a)
\]
for the sequence of functions \gls{OPCBk}$ \colon \A \to [0,1]$ defined recursively by $\OP{CB}_0(a)=w(a)$ for all $a \in \A$ and 
\[
\OP{CB}_{k+1}(a)=\frac{w(a)}{1+|\Att^*(a)|+ \tfrac{1}{|\Att^*(a)|}  \sum\limits_{b \in \Att^*(a)} \OP{CB}_{k}(b)}
\]
Here, \gls{Atts}$(a) = \{ b \in \Att(a) \mid w(b)>0\}$ and if $\Att^*(a) = \emptyset$ the term involving division by $|\Att^*(a)|$ vanishes.
\end{example}

\begin{example}\label{example:HC semantics}
The \textbf{weighted h-categorizer} gradual semantics \gls{SigmaHC}.
Given $\F=\langle \A,\D,w\rangle$, the acceptability degree of an argument $a \in \A$ is
\[
\Sigma_{\HC}^\AF(a) = \lim_{k \to \infty} \OP{HC}_k(a)
\]
for the sequence of functions \gls{OPHCk}$ \colon \A \to [0,1]$ defined recursively by $\OP{HC}_0(a)=w(a)$ and
\[
\OP{HC}_{k+1}(a)=\frac{w(a)}{1+\sum\limits_{b \in \Att(a)} \OP{HC}_{k}(b)}
\]  
\end{example}

A crucial issue to address with these definitions is the need to prove that the limits exist and yield scoring functions. The convergence has been shown to exist (and be unique) for the three semantics mentioned above by Amgoud et al. (see in particular Theorems 7, 12, and 17 in \cite{AMGOUD2022103607}).

Using gradual semantics, one can obtain a ranking on arguments by using their acceptability degrees.
With the three gradual semantics above, it is the case that $a$ is at least as preferred as $b$ iff $a$'s acceptability degree is at least as high as $b$'s.
We use the following notation to describe preferences over  arguments, where $a$ \gls{trianglerighteq} $b$ denotes that $a$ is at least as preferred as $b$, $a$ \gls{simeq} $b$ iff $ a \trianglerighteq b \wedge b \trianglerighteq a$, $a$ \gls{triangleright} $b$ iff $a \trianglerighteq b \wedge a \not\trianglelefteq b$, and $a$ \gls{trianglelefteq} $b$ iff $a \not \triangleright\ b$.

\begin{example} (taken from \cite{oren2022inverse}).  \label{ex:semantics}
Let $\AF = \langle \A, \D, w \rangle$ be a WAF depicted in Figure \ref{fig:ex1} with weights $w(a_0) = 0.43, w(a_1) = 0.39, w(a_2) = 0.92$, and $w(a_3) = 0.3$.
Table \ref{tab:ex1} lists the acceptability degrees and the associated rankings on arguments for the semantics $\Sigma_{\MB}, \Sigma_{\CB}$ and $\Sigma_{\HC}$ described in Examples \ref{example:MB semantics}, \ref{example:CB semantics} and \ref{example:HC semantics}.
%

\begin{figure}[h]
\centering
\begin{tikzpicture}[scale=0.1]
\tikzstyle{every node}+=[inner sep=0pt]
\draw [black] (18.5,-39.5) circle (3);
\draw (18.5,-39.5) node {$a_0$};
\draw [black] (32.2,-28.1) circle (3);
\draw (32.2,-28.1) node {$a_1$};
\draw [black] (32.2,-39.5) circle (3);
\draw (32.2,-39.5) node {$a_2$};
\draw [black] (45.8,-39.5) circle (3);
\draw (45.8,-39.5) node {$a_3$};

\draw [black] (21.5,-39.5) -- (29.2,-39.5);
\fill [black] (29.2,-39.5) -- (28.4,-39) -- (28.4,-40);
\draw [black] (30.877,-25.42) arc (234:-54:2.25);
\fill [black] (33.52,-25.42) -- (34.4,-25.07) -- (33.59,-24.48);
\draw [black] (32.2,-31.1) -- (32.2,-36.5);
\fill [black] (32.2,-36.5) -- (32.7,-35.7) -- (31.7,-35.7);
\draw [black] (33.523,-42.18) arc (54:-234:2.25);
\fill [black] (30.88,-42.18) -- (30,-42.53) -- (30.81,-43.12);
\draw [black] (42.8,-39.5) -- (35.2,-39.5);
\fill [black] (35.2,-39.5) -- (36,-40) -- (36,-39);
\end{tikzpicture}
\caption{Graphical representation of a WAF.}
\label{fig:ex1}
\end{figure}
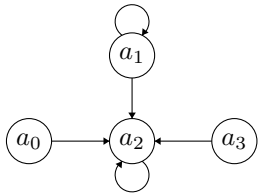

\begin{table}[ht]
    \centering
    \renewcommand{\arraystretch}{1.3}
    \begin{tabular}{|c|c|c|c|c|c|}
    \hline
         & $a_0$& $a_1$& $a_2$& $a_3$&  \\
    \hline
         $w(a_i)$  & 0.43 &  0.39 &  0.92 & 0.30 & Argument ranking\\
         \hline
         $\Sigma^\AF_{\MB}(a_i)$& 0.43 & 0.30 & 0.58 & 0.30 & $a_1 \simeq a_3 \triangleleft a_0 \triangleleft a_2$\\
         \hline
         $\Sigma^\AF_{\HC}(a_i)$& 0.43 & 0.30 & 0.38 & 0.30 & $ a_1 \simeq a_3 \triangleleft a_2 \triangleleft a_0$\\
         \hline
         $\Sigma^\AF_{\CB}(a_i)$& 0.43 & 0.18 & 0.17 & 0.30 & $ a_2 \triangleleft a_1 \triangleleft a_3 \triangleleft a_0$\\
         \hline
    \end{tabular}
    \caption{Acceptability degrees of the arguments from Figure \ref{fig:ex1}, rounded to 2 decimal numbers.}
    \label{tab:ex1}
\end{table}
\end{example}


We now make a historical note about the problem which inspired this paper.
The {\em inverse problem} seeks a set of initial weights which, when applied to a specific argumentation framework under the chosen semantics will result in a desired preference ordering, derived from the numeric acceptance degree computed for each argument.

In \cite{oren2022inverse} the authors examined how weights could be computed for a given acceptance degree. 
The approach described there consists of two phases. In the first phase, a target acceptance degree is computed for each argument. In the second phase, a numerical method (the bisection method)  is adapted to find the initial weights which lead to this acceptance degree. \cite{oren2022inverse} describe several heuristics which underpin the numerical method, and identify one which works well in practice.


The inverse problem for $\Sigma_{\HC},\Sigma_{\MB}$ and $\Sigma_{\CB}$ was the focus of previous research \cite{oren2022inverse}.
In Theorem \ref{theorem:explicit MB CB HC} below we give a complete characterisation of the inverse problem for these semantics.
In fact we give a positive answer to several other related problems, specifically the \emph{projective preference ordering problem}, two variants of the \emph{reflection problem}, and the \emph{radiality problem}.
These are described in detail next.

\section{Rephrasing Weighted Gradual Semantics with Scoring Schemes}
\label{subsec:scoring-schemes}

To more easily formalise the inverse problem and related problems, we start by presenting a different approach to weighted gradual semantics.

This approach revolves around the notion of \emph{scoring schemes} (Definition \ref{def:scoring scheme}), which are functions mapping from weights to acceptability degrees. Thus, a weighted gradual semantics can be seen as a function which takes as input a non-weighted argumentation framework, and returns a specific scoring scheme, which then allows us to map from argument weights in the framework to acceptability degrees (Definition \ref{def:absweighgradsem}).

Rather than proving results about individual gradual semantics, the use of scoring schemes allows us to reason about the behaviour of a family of gradual semantics which return  scoring schemes obeying some properties, enabling us to prove results for every gradual semantics in the family.

With the definition of scoring schemes --- and thus gradual semantics --- in hand, we turn our attention to inverse problems in Section \ref{subsec:inverse and related problems}. Such inverse problems revolve around moving from acceptability degrees to weights, but consider different facets of doing so. All ask whether some scoring scheme satisfies some property. In Section \ref{subsec:based scoring schemes} we identify a family of scoring schemes, referred to as based scoring schemes (see Section \ref{subsec:based scoring schemes}), which satisfies perhaps the most important of our desired properties, namely that any gradual semantics built using such a scoring scheme is guaranteed to converge to a unique vector of acceptability degrees given a vector of initial weights. In later sections, we show how further specialisations of based scoring schemes (and thus gradual semantics) satisfy the various inverse problems.

\subsection{Scoring Schemes and Abstract Weighted Gradual Semantics}

For any $A$ and $B$, let \gls{FuncAB} denote the set of all functions $f \colon A \to B$ and \gls{EndA} the set of endomorphisms (i.e., self functions) $f \colon A \to A$. Given $T \in \End(A)$, a sequence $a_0,a_1,\dots \in A$ is called a {\bf \gls{Tseq}} if $a_{k+1}=T(a_k)$ for all $k \geq 0$.
Given a directed graph $\G=\langle \A,\D \rangle$, we may, and will, assume throughout that $\A=\{1,\dots,n\}$ and $n \geq 1$.

A scoring function $s \colon \A \to [0,1]$ is identified with a vector $(x_1,\dots,x_n)$ in \gls{X}$=[0,1]^n$ which we call the {\bf space of scores}.
Similarly the {\bf space of weights}, i.e., functions $w \colon \A \to [0,1]$, can be identified with vectors in \gls{W}$=[0,1]^n$.
We also set \gls{K}$=[0,1]^n$ which we call the {\bf space of vertices}\footnote{This terminology arises due to an intuitive interpretation of Definition \ref{def:weighted scoring base}, and a further discussion of this term is provided at that point.}.
We regard $X,W$, and $K$ as subspaces of the Euclidean space $\RR^n$, which we equip with the partial order $x \preceq y$ if $x_i \leq y_i$ for all $1 \leq i \leq n$. Thus, they also inherit this order\footnote{Note that although $X,W$ and $K$ all refer to $[0,1]^n$, we use distinct labels for these to highlight their different roles.}.

\begin{defn}\label{def:scoring scheme}
A {\bf scoring scheme} is a function \gls{sigma}$ \colon W \to X$.
The space of all scoring schemes is \gls{Scheme}$ = \Func(W,X)$.
\end{defn}

The space of {\bf acceptability degrees} of $\sigma \in \Scheme$ is \gls{Dsigma}$ = \{ \sigma(w) | w \in W \} = \sigma(W)$, the image of $W$.

We emphasise that a scoring scheme is not a gradual semantics. Instead, it simply maps from a vector of initial weights to acceptability degrees. Thus, every line in Table \ref{tab:ex1} is a scoring scheme (with the first line being the identity scoring scheme) produced by a weighted gradual semantics for the argumentation framework illustrated in Figure \ref{fig:ex1}.

Let $\Sigma$ be a weighted gradual semantics.
For any argumentation framework $\G=\langle \A,\D \rangle$, we obtain a scoring scheme $\sigma^\G \colon W \to X$ 
\[
\sigma^\G (w) = \Sigma^{\langle \A,\D,w\rangle}
\] 

This means that a gradual semantics, which takes a weighted argumentation framework and returns a scoring function, can be seen as a function which takes an argumentation framework (without weights) and return a particular scoring scheme. We formalise this in the next explanatory comment.

\begin{scholiumx}
A weighted gradual semantics $\Sigma$ can be viewed as a function 
\[
\Sigma \colon \{ \text{argumentation frameworks $\G$}\} \xto{\ \G \mapsto \sigma^\G\ } \{ \text{scoring schemes}\}.
\]

By partitioning argumentation frameworks according to the number of their arguments $n=|\A|$ and using the notation above, a weighted gradual semantics is a collection of functions, one for each $n \geq 1$
\[
\Sigma \colon \{ \text{argumentation frameworks $\G$ with $|\A|=n$} \} \xto{\G \mapsto \sigma^\G} \Scheme.
\]
The acceptability degree space of the semantics $\Sigma$ for a graph $\G$ is $D_{\sigma^\G}$.
\end{scholiumx}

Since every argumentation framework $\G$ is determined by its adjacency matrix, a weighted gradual semantics is the same as a collection of functions, one for each $n \geq 1$,
\[
\Sigma \colon \{ \text{ $n \times n$ adjacency matrices} \} \to \Scheme.
\]

Let \gls{MatnnnR} be the set of all square matrices of size $n$ whose elements are  in $\RR$. If $A \in \Mat_{n \times n}(\RR)$ we write $A \geq 0$ if $a_{i,j} \geq 0$.
Denote \gls{MatnnnR+} $= \{ A \in \Mat_{n \times n}(\RR) : A \geq 0 \}$.

\begin{defn}
\label{def:absweighgradsem}
An {\bf \gls{awgs}} is a collection of functions, one for each $n \geq 1$,
\[
\Sigma \colon \Mat_{n \times n}^+(\RR) \to \Scheme
\]
\end{defn}

Thus, by restriction to $n \times n$ adjacency matrices (those whose entries are only zeroes and ones, and where a 1 entry at index $i,j$ denotes an attack from the argument indexed by $j$ to the argument indexed by $i$) any {\em abstract} weighted gradual semantics restricts to an ordinary weighted gradual semantics.

\begin{example}
    Let $\G$ be a complete graph with arguments $\A = \{1,2,3\}$. Then, the acceptability degree space of semantics 
    $\Sigma_{\HC}$ for $\G$ is $D_{\sigma^{\G}}$, where $\sigma^\G$ is the associated scoring scheme and for every $w \in W, \sigma^\G(w) = \Sigma_{\HC}^{\langle \A, \D, w \rangle}$. A representation of $D_{\sigma^{\G}}$ is shown in Figure \ref{fig:D_sigma}. Note that the boundaries of the acceptability degree space are not linear.

    \begin{figure}
        \centering
        \includegraphics[width=5cm]{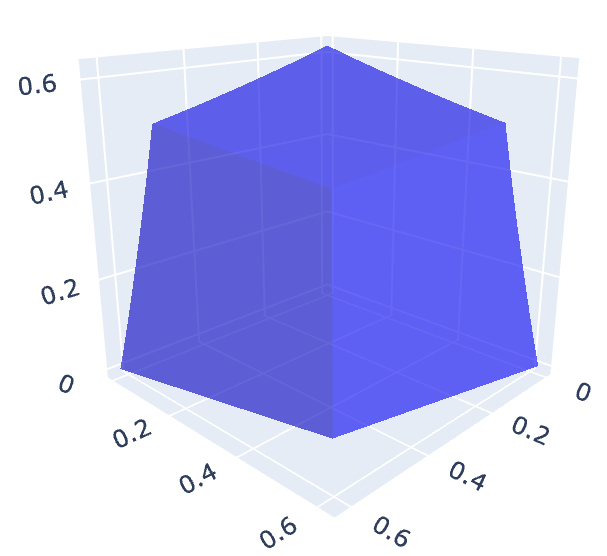}
        \caption{Representation (in blue) of the acceptability degree space of $\Sigma_{\HC}$ for a complete argumentation graph with 3 arguments. Note that the surfaces in the figure are non-linear.}
        \label{fig:D_sigma}
    \end{figure}
\end{example}

\subsection{The inverse problems}
\label{subsec:inverse and related problems}

We generalise the inverse problem introduced in \cite{oren2022inverse} to a family of \emph{inverse-type problems}. These revolve around moving from acceptability degrees to weights, but consider different facets in doing so. Thus, we are --- for example --- interested in finding conditions on scoring schemes 
which guarantee that acceptability degrees satisfying certain properties are realised, and in this case calculate the weights that realise them.
We formulate four inverse-type problems that are the focus of this paper.

The first problem asks whether, for a given scoring scheme $\sigma$, we can find an \textit{inverse function} \gls{inv} such that for any element $x \in X$, we can determine whether $x$ is in the space of acceptability degrees of $\sigma$ by testing whether $\inv(x) \in W$, and in this case $w=\inv(x)$ is a weight such that $\sigma(w)=x$.

\begin{problem}\label{problem:inverse}
The {\bf \gls{invP}} for a scoring scheme $\sigma \in \Scheme$ seeks  
a (computable) function $\inv \colon X \to \RR^n$ such that 
\begin{enumerate}[label=(\roman*)]
\item 
$x \in D_\sigma \iff \inv(x) \in W$,  and in this case

\item
$\sigma(\inv(x))=x$.

\end{enumerate}

\end{problem}

\begin{example}[Example \ref{ex:semantics} cont.]
In section \ref{subsec:Lp lambda mu A}, we will show that the scoring scheme $\Sigma^\AF_{\CB}$ has a positive solution to this problem. Thus, given the final acceptability degrees shown in Table \ref{tab:ex1} for the argumentation graph shown in Figure \ref{fig:ex1}, it is possible to identify initial weights (e.g., [0.43, 0.39, 0.92, 0.3]) which result in these acceptability degrees ([0.43, 0.3, 0.38, 0.3]).
\end{example}

The second problem (reflection) ask whether the scoring scheme is injective or a homeomorphism onto its image.
This allows us to answer questions such as: ``Given an argumentation framework and a gradual semantics $\Sigma$, is it possible to find two distinct weighting functions such that the acceptability degrees will be the same?''

\begin{problem}\label{problem:reflection}
The {\bf  \gls{reflP}} for  $\sigma \in \Scheme$ asks whether $\sigma \colon W \to X$ is injective, i.e., every acceptability degree is obtained by a {\em unique} weight $w \in W$.

The closely related {\bf \gls{toporflP}} asks whether $\sigma \colon W \to \D_{\sigma}$ is a \gls{homeomorphism}\footnote{Recall that a homeomorphism is a function which is bijective and continuous, and whose inverse is also continuous. } onto its image.

\begin{example}[Example \ref{ex:semantics} cont.]
In section \ref{subsec:Lp lambda mu A}, we will show that the scoring scheme $\Sigma^\AF_{\CB}$ has a positive solution for the  reflection problem (and thus the reflection problem). Thus, given the final acceptability degrees shown in Table \ref{tab:ex1} for the argumentation graph shown in Figure \ref{fig:ex1}, it is possible to state that the initial weights [0.43, 0.39, 0.92, 0.30] are the only ones which result in these acceptability degrees ([0.43, 0.30, 0.38, 0.30]). 

Moreover, since the topological reflection problem is also satisfied, a small change to the initial weights will result in a small change in the acceptability degrees (and vice versa). Thus, for example, the final acceptability degrees ([0.4301,0.29995035,0.38139473,0.301]) is obtained from the initial weights ([0.4301,0.3899,0.92,0.301]).
\end{example}

%
\end{problem}

The third problem asks whether any element $y$ of the space of scores can be projected into an element $x \in \D_\sigma$ by re-scaling it.
Note that the acceptability degrees $x$ have the same arguments' preference ordering and ratio between them as the acceptability degrees $y$.
This will allow us to answer questions such as: ``Given an argumentation framework and a gradual semantics $\Sigma$, is it possible to find a weighting function such that the acceptability degree of argument $a$ is $n$ times as much as the one of $b$?''. 
If the projective preference ordering problem is answered in the affirmative, one can easily find acceptability degrees for each argument that are ratios of each other (e.g., argument $a_1$ should be 0.5 the acceptability degree of $a_2$, which should be 3 times $a_3$, etc). This is achieved by creating arbitrary values for each argument which obey the ratio, and then seeking out an appropriate $t$ so as to scale these down to a realisable element of $\D_\sigma$. This can be seen as a generalisation of \cite{oren2022inverse}, which identified a specific ratio of acceptability degrees for each argument to ensure achievability.

\begin{problem}\label{problem:preference}
The {\bf \gls{ppop}} for a scoring scheme $\sigma$ asks whether for any $y \in X$ there exists some $t>0$ such that $t y \in D_\sigma$.

\end{problem}

\begin{example}[Example \ref{ex:semantics} cont.]
\label{ex:projective_pref}
In Section \ref{subsec:Lp lambda mu A}, we will show that the scoring scheme $\Sigma^\AF_{\CB}$ has a positive solution for the projective preference ordering problem. Let us consider the case where one may want $a_2$ to have an acceptability degree that is two times higher than all the other arguments. The target acceptability degree $x = [0.5, 0.5, 1, 0.5]$ is not valid (as there are no weighting functions that realise this acceptability degree vector). However, the scaled down vector $t x$, with $t = 0.40$ is valid and obtained using the initial weights [0.20, 0.24, 0.80, 0.20].
\end{example}

The fourth problem ask whether for any element $x$ of the acceptability degree space of $\sigma$ ($D_\sigma)$, all the elements on the line from $0^n$ to $x$ are in $D_\sigma$.  In other words, a positive answer to this question means there are no ``gaps'' in the acceptability degree space along any line starting at the origin and a point in the acceptability degree space. This is important as it allows us to scale down any element along the line (via Problem \ref{problem:preference}).

%

\begin{problem}\label{problem:radiality}
The {\bf \gls{radialP}} for a scoring scheme $\sigma$ asks whether for any $x \in D_\sigma$, the line segment $[0,x]=\{tx:0 \leq t \leq 1\} $ in $\RR^n$ is contained in $D_{\sigma}$.
In this case we call $D_\sigma$ {\em radial}.
%
\end{problem}


\begin{example}[Example \ref{ex:projective_pref} cont.]
In Section \ref{subsec:Lp lambda mu A}, we will show that the scoring scheme $\Sigma^\AF_{\CB}$ has a positive solution for the radiality problem. As mentioned in Example \ref{ex:projective_pref}, the initial weights [0.20, 0.24, 0.80, 0.20] leads to the acceptability degree vector $0.4 x$. Thus, the acceptability degree vector $0.35 x$  is also valid and is obtained using the initial weights [0.175, 0.205625, 0.65625, 0.175] (and indeed, any acceptability degree vector for which $t$ approximately less than or equal to 0.463).
\end{example}

\begin{defn}
We say that a weighted gradual semantics $\Sigma$ has a solution for the inverse problem \ref{problem:inverse}, the (topological) reflection problem \ref{problem:reflection}, the projective preference ordering problem \ref{problem:preference} or the radiality problem \ref{problem:radiality} if for any argumentation framework $\G$ the scoring scheme $\sigma^\G =\Sigma(\G) \in \Scheme$ has a solution for these problems.
\end{defn}

\subsubsection*{The importance of inverse problems} 
We briefly segue to discuss the importance of these four problems. The clearest application of these problems is in opponent modelling and preference elicitation \cite{DBLP:journals/corr/abs-2211-16118,DBLP:conf/prima/MahesarOV20,oren2022inverse}. During dialogue, participants often express acceptability degrees rather than initial weights of arguments. By being aware of the argumentation framework and the argument initial weights, one can more easily compute the acceptability degree of arguments one may want to introduce and determine whether doing will, or won't contribute to dialogical goals. Our work allows one to identify these (hidden) initial weights from (dialogically exposed) acceptability degrees.
We now delve more deeply into the importance of, and potential applications, for each problem.

\paragraph{The inverse problem} This problem is perhaps the most fundamental, and allows us to determine whether it is possible to move from some acceptability degrees to valid initial weights. Thus, if a gradual semantics satisfies the inverse problem, then  given a target set of acceptability degrees for some arguments, one can determine whether these are achievable, i.e., whether there exist valid initial weights which lead to them. 

In a dialogue where participants disagree about final acceptability degrees, achievable initial weights suggest that --- for example --- a dialogue participant must seek more evidence to support the desired initial weights. In cases where the desired initial weights are invalid, the participant understands that the underlying argumentation graph must be modified, for example through the addition or deletion of attacks or arguments (c.f., \cite{DBLP:conf/ijcai/Coste-MarquisKM15}).

\paragraph{The (topological) reflection problem} The inverse problem does not guarantee a one-to-one mapping between acceptability degrees and initial weights. Thus, when trying to learn the initial weights of a reasoner, one may identify the incorrect weights given some acceptability degree. A positive solution to the reflection problem provides this guarantee. By having a positive answer to both of the reflection and  inverse problem, we can ensure that our semantics can be used to learn the initial weights used by a reasoner, allowing us to infer this element of their knowledge. 
A positive answer to the topological reflection problem (which also results in a positive answer to the reflection problem) ensures that a small variation in the initial weights induces a small variation of the acceptability degree. This is important if one were to consider the impact of arguments \cite{DBLP:conf/ecsqaru/DelobelleV19,DBLP:journals/corr/abs-2401-08879}. Without this property, a small change in initial weights could lead to large changes in acceptability degrees (and vice-versa), making the analysis of argument impact challenging. A semantics which satisfies topological reflection is thus --- in a sense --- well behaved and more amenable to impact analysis.

\paragraph{The projective preference ordering problem} The ability to achieve ratios of acceptability degrees underpins the approach described in \cite{oren2022inverse} to find initial weights which could be used for any semantics which satisfies this problem. In addition, this problem suggests an approach to argument weight revision; while argument graphs are normally modified to achieve some target argument values (e.g., justification status \cite{jeanguy13dynamic} for classical frameworks, and preferences or acceptability degrees in the case of argumentation semantics \cite{tarle22multiagent}), any semantics using this property can instead be scaled to achieve the same effect.

\paragraph{The radiality problem} While the projective preference ordering ensures that one can scale down any vector of target acceptability degree to achieve a valid one, there is no guarantee that further scaling down will also yield valid answers. A positive answer to this problem provides such a guarantee, as well as providing a necessary condition for the volume of acceptability degrees in to contain no gaps. In cases where initial weights represent some form of certainty or confidence, and where evidence must be gathered (by expending resources) to strengthen these initial weights, the ability to weaken arguments arbitrarily would allow a reasoner to trade off resource consumption with argument strength (similar to the ideas described in \cite{skitalinskaya-etal-2023-claim}). 

\medskip

One of the achievements of this paper is the identification of a large family of abstract weighted gradual semantics which includes the semantics $\Sigma_{\HC},\Sigma_{\MB}$ and $\Sigma_{\CB}$ from Examples \ref{example:MB semantics}, \ref{example:CB semantics}, and \ref{example:HC semantics} and for which all four inverse-type problems above are solved.
Moreover, we give explicit and easily computable inverse functions $\inv_{\MB}, \inv_{\CB}$ and $\inv_{\HC}$ for these three semantics.
See Theorem \ref{theorem:explicit MB CB HC} below for a detailed statement of the result.

\subsection{Based scoring schemes}
\label{subsec:based scoring schemes}

To compute the acceptability degrees for an argumentation framework,the weighted gradual semantics which we consider in this work require us to compute a fixed point through an iterative calculation. In this section, we capture this iterative process. Following some preliminary notions, we begin with the definition of a scoring base (Definition \ref{def:scoring base}) which captures some argument properties and specifies how the acceptability degrees within the argumentation framework are aggregated across iterations. A scoring dynamics (Definition \ref{def:scoring dynamic}) describes how acceptability degrees can be updated across each iteration, and guarantees that the iterative process will converge. We link these two concepts by showing that a scoring dynamics can be obtained from a scoring base. Given a specific element of the space of weights, we describe a weighted scoring base (Definition \ref{def:weighted scoring base}), which associates a scoring base to this element. Thus, a scoring scheme can be defined by specifying a weighted scoring base and induced scoring dynamics. 

Critically, we show that any weighted gradual semantics which only produces scoring schemes which can be described using a weighted scoring base (and concomitant induced scoring dynamics) will converge, generalising the important result of \cite{pu_argument_2014}. In Section \ref{subsec:inverse problems}, we describe additional restrictions on weighted scoring bases which must be satisfied for the various inverse problems to have a positive solution. In Section \ref{subsec:Lp lambda mu A}, we identify a class of based scoring schemes which meet all of these restrictions.

%

We first start with the necessary definitions to introduce the notion of a scoring base.

\begin{defn}\label{def:BHI}
A function $f \colon \RR^n \to \RR$ is called {\bf \gls{increasing}} if  for any $x,y \in \RR^n$
\[
x \preceq y \implies f(x) \leq f(y).
\]
It is called {\bf \gls{homogeneous}} if for all $x \in \RR^n$ and all $t \geq 0$
\[
f(t x)=t  f(x).
\]
A function $f \colon X \to \RR^n$ is called homogeneous (resp. increasing) if each component $f_i \colon X \to \RR$ is homogeneous (resp. increasing).
We say it is \emph{\gls{non-negative}}, written $f \geq 0$, if $f_i(x) \geq 0$ for all $1 \leq i \leq n$.
It is \emph{\gls{bounded}} if there is some $M>0$ such that $f_i(x) \leq M$ for all $1 \leq i \leq n$.
Let \gls{BpHIn} denote the set of all functions $f \colon X \to \RR^n$ which are bounded, non-negative, homogeneous and increasing.
For $n=1$, we write \gls{BpHI}.
\end{defn}

The following two propositions show that the collection $\BpHI(X)$ is very rich, and underpins our notion of a scoring base (Definition \ref{def:scoring base}). In turn, these scoring bases are used as the building blocks of the existing gradual semantics studied in the paper (see Example \ref{ex:scoringdynamics}) and new (Section \ref{sec:new_semantics}) semantics.
%

\begin{proposition}\label{prop:HM linear}
$\BpHI(X)$ contains all projections $\pi_i \colon X \xto{x \mapsto x_i} [0,\infty)$. 
\end{proposition}

We can combine elements of $\BpHI(X)$ to obtain more complex functions which are also elements of $\BpHI(X)$.

\begin{proposition}\label{prop:HM closure}
Given a finite $k$, let $\psi_1,\dots,\psi_k \in \BpHI(X)$.
Then the following functions obtained from $\psi_1,\dots,\psi_k$ are also in $\BpHI(X)$.
\begin{enumerate}[label=(\roman*)]
\item
$\sum_{j=1}^{k} a_j \psi_j$   
for any $a_1,\dots,a_k \geq 0$. 

\item
$\max\{\psi_1,\dots,\psi_k\}$. 
\item 
$\sqrt[k]{\psi_1 \cdots \psi_k}$ (geometric mean).

\item 
$\sqrt[p]{\psi_1^p + \cdots +\psi_k^p}$ for any $p>0$.
\end{enumerate}
\end{proposition}

Proposition \ref{prop:HM closure} applies to $\BpHI(X)$, but is easily generalised to $\BpHI_n(X)$. 

%

We can now introduce the fundamental notion of a scoring base.

\begin{defn}\label{def:scoring base}  
A {\bf \gls{scoringbase}} is a pair $(c,f) \in K \times \BpHI_n(X)$.
That is,
\begin{itemize}
\item
$c \in K$ and 
\item
$f \colon X \to [0,\infty)^n$ is bounded, homogeneous, non-negative, and increasing.
\end{itemize}
The set of all scoring bases \gls{BaseX} denotes 
\[
 K \times \BpHI_n(X) \qquad  (\subseteq K \times \Func(X,\RR^n)). 
\]
\end{defn}

From a scoring base, one can construct a particular endomorphism from $X$ to $X$.

\begin{defn}\label{defn:T_kappa phi}
Let $(c,f)$ be a scoring base, and $i \in \{1 \ldots |n|\}$, be an index w.r.t arguments.
Define  \gls{Tcf}$ \in \End(X)$ by
\[
T_{(c,f)}(x)_i = \frac{c_i}{1+f(x)_i}.
\]

$T_{(c,f)}$ is well defined because $0 \leq c_i \leq 1$ and $f(x)_i \geq 0$ so $T_{(c,f)}(x) \in X$ for any $x \in X$.
\end{defn}

A crucial question is: ``Given an arbitrary element $x \in X$, will the successive application of $T_{(c,f)}$ lead to convergence and will it always converge to the same fix point?''.
To formalise this, we introduce the following notion of scoring dynamics so that the previous question can thus be rephrased as: ``Is $T_{(c,f)}$ a scoring dynamics?''.

\begin{defn}\label{def:scoring dynamic}
A {\bf \gls{scoringdynamics}} is a function $T \in \End(X)$ such that
\begin{enumerate}[label=(\roman*)]
\item
$T$ has a unique fixed point denoted $\fix(T) \in X$, and
\item
Every $T$-sequence $x_0,x_1,\dots \in X$ is convergent to $\fix(T)$.
\end{enumerate}
Denote by \gls{SDX} the collection of all scoring dynamics.
The assignment $T \mapsto $\gls{fixT} gives a function 
\[
\fix \colon \SD(X) \to X.
\]
\end{defn}


To answer the question ``Is $T_{(c,f)}$ a scoring dynamics?'', we start by showing, in Theorem \ref{th:pu-generalisation}, that many order reversing functions of $\End(X)$ are scoring dynamics.

While our proof for Theorem \ref{th:pu-generalisation} is inspired by the work of Pu, Zhang and Luo (\cite[Theorems 1 and 2]{Pu14}) who showed that a unique fixed-point exists for the h-categorizer semantics, i.e.,\ in the case where all arguments have an initial weight of 1; we are able to prove such a result for a much larger class of functions.

\begin{definition}
A function $f \colon X \to \RR^n$ is called {\bf \gls{orderpreserving}} if $x \preceq x' \implies f(x) \preceq f(x')$ for any $x,x' \in X$.
It is called {\bf \gls{orderreversing}} if $x \preceq x' \implies f(x') \preceq f(x)$ for any $x,x' \in X$.    
\end{definition}

\begin{theorem}\label{th:pu-generalisation}
Let $f:X \to X$ be a function such that
\begin{enumerate}[label=(\alph*)]
\item \label{th:pu:reversing}
$f$ is order reversing, and

\item \label{th:pu:bound}
there is some $0 < \alpha \leq 1$ such that for any $x \in X$ and any $0 \leq t \leq 1$
\[
f(tx) \preceq \frac{1}{t+\alpha(1-t)} f(x).
\]
\end{enumerate}
Then $f$ has a unique fixed point $y \in X$. 
Furthermore, $y = \lim_{k \to \infty} x^{(k)}$ where $x^{(k)}$ is any sequence defined recursively by choosing $x^{(0)} \in X$ arbitrarily and $x^{(k+1)}=f(x^{(k)})$.
\end{theorem}

The next theorem shows that this generalisation applies to $T_{(c,f)}$ (obtained via a scoring base), and we will show below how this therefore applies to a very large class of gradual semantics (see Definition $\ref{def:abstract-weighted-based}$).

\begin{theorem}\label{theorem:T_c f in Base(X)}
Let $(c,f) \in \Base(x)$ be a scoring base.
Then $T_{(c,f)} \in \SD(X)$.
\end{theorem}

A weighted scoring base is a function which associate a scoring base to each element of $W$ by making use of two functions which instantiate each of the elements of the scoring base.

\begin{defn}\label{def:weighted scoring base}
A {\bf \gls{wsb}} is a function $b \colon W \to \Base(X)$.
More specifically, $b=$\gls{kv} where \gls{kappa} and \gls{varphi} are functions
\begin{eqnarray*}
&&\kappa \colon W \to K \\
&&\varphi \colon W \to \Func(X,\RR^n) 
\end{eqnarray*}
such that $(\kappa(w),\varphi(w)) \in \Base(X)$ for all $w \in W$, that is $\varphi \colon W \to \BpHI_n(X)$.
The set of all weighted scoring bases is denoted \gls{wBaseX}.

The associated scoring scheme \gls{sigmab}$ \in \Func(W,X)$  is defined by 
\[
\sigma_b(w) = \fix(T_{(\kappa(w),\varphi(w))}).
\]
That is, $\sigma_b$ is the composition from $X$ to $X$ via a scoring base and scoring dynamics.
\[
\sigma_b \ \colon \  W \xto{b} \Base(X) \xto{T} \SD(X) \xto{\fix} X.
\]
We will also  write $\sigma_{(\kappa,\varphi)}$ for $\sigma_b$.
We can thus obtain a function
\[
\wBase(X) \xto{ \ (\kappa,\varphi) \mapsto \sigma_{(\kappa,\varphi)} \ } \Scheme.
\]
A scoring scheme is {\bf based} if it has the form $\sigma_{(\kappa,\varphi)}$ for some weighted scoring base $b=(\kappa,\varphi)$.
Its acceptability degree space is denoted $D_{(\kappa,\varphi)}$.
\end{defn}

Figure \ref{fig:sectino-3.1} summarises the notions and the results of Section \ref{subsec:based scoring schemes}.

\begin{figure}[!ht]
    \centering
    \includegraphics[width=13.7cm]{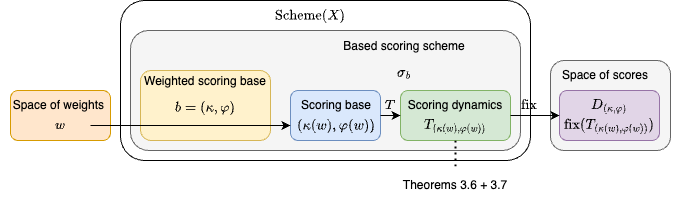}
    \caption{Representation of a based scoring scheme $\sigma_b$ and the contributions of  Section \ref{subsec:based scoring schemes}.}
    \label{fig:sectino-3.1}
\end{figure}


At this point, it should be clear why we referred to $K$ as the space of vertices. From the above definition, we can see that the acceptability degree (the fixed point) of this weighted semantics lie in a ``box'' inside the space of scores $X = [0,1]^n$ whose \textit{extreme vertices} are $(0,...,0)$ and $\kappa(w)$. Thus,  $\kappa(w)$ can be viewed as the \textit{vertex} of the semantics and the set of all possible vertices, i.e $K=[0,1]^n$, is called the space of vertices.

\begin{defn}
\label{def:abstract-weighted-based}
An {\bf \gls{awbgs}} is one which is obtained by composing a function
\[
\Mat_{n \times n}^+(\RR) \xto{\beta} \wBase(X)
\]
with the map $\wBase(X) \xto{b \mapsto \sigma_b} \Scheme$ in Definition \ref{def:weighted scoring base}.
\end{defn}

By construction, if $\Sigma$ is an abstract weighted based gradual semantics then for any $A$ the scoring scheme $\sigma(A)=\sigma_{\beta(A)}$ (Definition \ref{def:weighted scoring base}) is based.

\begin{example}[Example \ref{ex:semantics} cont.]
\label{ex:scoringdynamics}
In this example, we will show how the different definitions from Section \ref{subsec:scoring-schemes} are used with respect to the weighted gradual semantics on a concrete example.

From Example \ref{ex:semantics}, we have $n=4$ and $w = [0.43, 0.39, 0.92, 0.3] \in W$. The underlying argumentation graph can be encoded using the following adjacency matrix $A$.

$$ A= \begin{bmatrix}
0 & 0 & 0 & 0 \\
0 & 1 & 0 & 0 \\
1 & 1 & 1 & 1 \\
0 & 0 & 0 & 0 
\end{bmatrix}$$

The acceptability degrees for the   $\Sigma_{\HC}$ semantics are then shown in Table \ref{tab:ex1}. For this example, we can construct the scoring base as follows. Let $\kappa_1: W \to [0,1]^4$ be the identity and $\varphi_1: W \to \Func(X, [0, \infty)^4)$ defined as $\varphi_1(w)(x)_i = \sum\limits_{j=1}^4 a_{i,j}x_j$. The weighted scoring-base $(\kappa_1, \varphi_1)$ then produces the scoring base $(\kappa_1(w), \varphi_1(w)) \in \Base(X)$.

For this example, the corresponding scoring dynamics is

\[ T_{(\kappa_1(w), \varphi_1(w))}(x)_i = \frac{w_i}{ 1+ \sum\limits_{j=1}^4 a_{i,j}x_j} \]

Finally, $\fix(T_{(\kappa_1(w), \varphi_1(w))}) = [0.43, 0.3, 0.38, 0.30]$ is the output of the scoring scheme for $w$.

\end{example}

In this section, we introduced scoring schemes and linked these to abstract weighted gradual semantics. We also introduced our four inverse problems which give desirable properties that such scoring schemes should ideally satisfy. Finally, we considered based scoring schemes which move us closer to the gradual semantics in the literature by taking fixed-point computations into account.

\section{Solving Inverse Problems}
\label{sec:solving-problems}

\label{subsec:inverse problems}

In this section, we give constraints on weighted scoring bases $b$ so that the associated scoring scheme $\sigma_b$ admits a solution for the inverse problems described in Section \ref{subsec:inverse and related problems}.

We start by introducing the necessary definitions used to describe our constraints. Theorems \ref{theorem:inverse reflection for based schemes}--\ref{theorem:abstract cone solution} then identify the conditions that must be placed on weighted scoring bases for the associated scoring scheme to solve the various inverse problems. In Section \ref{subsec:Lp lambda mu A}, we propose a family of weighted gradual semantics which returns scoring schemes that always solve all our inverse problems.

The {\bf support} of $x \in \RR^n$ is the subset of $\{1,\dots,n\}$ defined as
\gls{supp}$ = \{ 1 \leq i \leq n : x_i \neq 0\}$.

\begin{defn}\label{def:preserve supports}
We say that a function $f \colon \RR^n \to \RR^n$ {\bf \gls{pressupport}} if for every $x \in \RR^n$
\[
\supp(x)=\supp(f(x)) 
\]
We say that $f \colon \RR^n \to Y$, for some set $Y$, is {\bf \gls{independentsupports}} if 
\[
\supp(x)=\supp(y) \implies f(x)=f(y).
\]
In this case, if $I \in \wp(\{1,\dots,n\})$ we write  \gls{fI}$\in Y$ for the value of $f$ on all $x \in X$ with support $I$.
\end{defn}

As we will see in the following theorems, some of the essential constraints for a weighted scoring base $(\kappa, \varphi)$ to have a solution for the aforementioned inverse problems are that $\kappa$ preserves supports and that $\varphi$ is independent of supports.

\begin{defn}\label{def:discerning inverse}
Let $A$ be a subset of $\RR^n$.
A {\bf \gls{dri}} for $g \colon A \to X$ is a function $r \colon X \to \RR^n$ such that $r(x) \in A \iff x \in g(A)$ and in this case $g(r(x))=x$.
That is, $g \circ r|_{g(A)}=\OP{id}_{g(A)}$.
\end{defn}

\begin{figure}
    \centering
    \includegraphics[width=14cm]{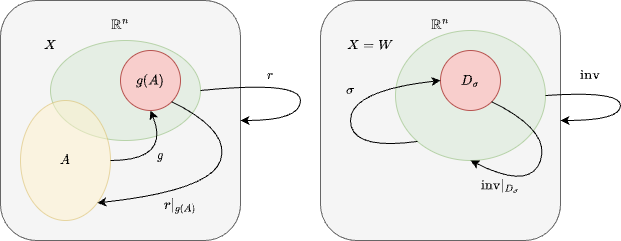}
    \caption{Representation of the discerning right inverse $r$ of $g$ (left) and a discerning right inverse of a scoring scheme $\sigma$ (right).}
    \label{fig:discerning_rgitht}
\end{figure}

\begin{remarkx}
The inverse problem \ref{problem:inverse} for a scoring scheme $\sigma$ seeks a computable discerning right inverse $\inv$ for $\sigma$ (see Figure \ref{fig:discerning_rgitht}).
\end{remarkx}

The next theorem gives conditions on  a weighted scoring base $(\kappa,\varphi) \in \wBase(X)$ so that the associated based scoring scheme $\sigma_{(\kappa,\varphi)}$ admits a solution to the inverse and the reflection problems (\ref{problem:inverse}, \ref{problem:reflection}).
It also gives conditions for it to admit a solution to the topological reflection problem.

\begin{theorem}\label{theorem:inverse reflection for based schemes}
Let $b=(\kappa,\varphi)$ be a weighted scoring base.
Assume that
\begin{enumerate}[label=(\roman*)]
\item
$\kappa \colon W \to K$ is the restriction of $\tilde{\kappa} \colon \RR^n \to \RR^n$ which is bijective and  preserves supports.

\item
$\varphi \colon W \to \Func(X,\RR^n)$ is independent of supports. 
\end{enumerate}
Then $\sigma_{(\kappa,\varphi)} \colon W \to X$ solves the inverse and reflection problems.
That is, $\sigma_{(\kappa,\varphi)}$ is injective and admits a discerning right inverse $\inv \colon X \to \RR^n$ which is defined as follows.
Recall the notation of $\supp$ from Definition \ref{def:preserve supports} and set
\begin{eqnarray*}
&& \overline{\inv}(x)_i = x_i(1+\varphi^{\supp(x)}(x)_i) \\
&& \inv(x) = \tilde{\kappa}^{-1} \circ \overline{\inv}.
\end{eqnarray*}
If $\tilde{\kappa}$ is a homeomorphism and $\varphi \colon W \to \Base(X)$ is constant with value $(c,f)$ such that $f \colon X \to [0,\infty)^n$ is continuous, then $\sigma_{(\kappa,\varphi)} \colon W \to X$ is a homeomorphism onto its image.
That is, $\sigma_{(\kappa,\varphi)}$ solves the topological inverse problem.
\end{theorem}

Our next goal is to find conditions on a weighted scoring base $(\kappa,\varphi)$ so that the projective preference ordering problem \ref{problem:preference} admits a solution.
We view $K=[0,1]^n$ as a subspace of $\RR^n$.
Let $0 \in K$ denote the origin in $\RR^n$.
Recall that a neighbourhood $U$ of $0$ in $K$ is the intersection of $K$ with an open subset of $\RR^n$ containing $0 \in \RR^n$.
Thus, there exists $\epsilon>0$ such that $U$ contains any $x \in \RR^n$ such that $\|x\| < \epsilon$ and $x_i \geq 0$ for all $1 \leq i \leq n$.

\begin{theorem}\label{theorem:abstract projective preference}
Let $b=(\kappa,\varphi)$ be a weighted scoring base.
Suppose that 
\begin{enumerate}[label=(\roman*)]
\item
$\kappa \colon W \to K$ preserves supports.

\item  
$\kappa(W) \subseteq K$ contains a neighbourhood $U$ of $0 \in K$.

\item 
$\varphi$ is independent of supports.
\end{enumerate}
Then the projective preference ordering problem \ref{problem:preference} has a solution for the scoring scheme $\sigma_{(\kappa,\varphi)}$.
\end{theorem}

Finally, we find conditions on a weighted scoring base $(\kappa,\varphi)$ which guarantees that the scoring scheme $\sigma_{(\kappa,\varphi)}$ admits a solution to the radiality problem (Problem \ref{problem:radiality}), i.e $D_{(\kappa,\varphi)}$ is radial subset of $X$.
Recall the partial order $\preceq$ on $\RR^n$.

\begin{defn}\label{def:preceq supp closed}
A subset $A$ of $K=[0,1]^n$ is {\bf \gls{preceqclosed}} in $K$, respectively {\bf \gls{supppreceqclosed}} in $K$, if for any $a \in A$ and any $x \in K$, 
\begin{align*}
x \preceq a &\implies x \in A
\\
x \preceq a \text{ and } \supp(x)=\supp(a) &\implies x \in A.
\end{align*}
\end{defn}

\begin{theorem}\label{theorem:abstract cone solution}
Let $b=(\kappa,\varphi)$ be a weighted scoring base.
Suppose that
\begin{enumerate}[label=(\roman*)]
\item
$\kappa \colon W \to K$ preserves supports.

\item
$\kappa(W)$ is $(\supp,\preceq)$-closed in $K$.

\item
$\varphi \colon W \to [0,\infty)^n$ is independent of supports.
\end{enumerate}
Then $\sigma_b$ solves the radiality problem, i.e $D_{(\kappa,\varphi)}=\sigma_{(\kappa,\varphi)}(W)$ is radial.
\end{theorem}

By Theorems \ref{theorem:inverse reflection for based schemes}, \ref{theorem:abstract projective preference} and \ref{theorem:abstract cone solution} all four Problems \ref{problem:inverse}--\ref{problem:radiality} have positive solution for particular based scoring scheme.
Now that we have defined the necessary constraints on weighted scoring bases for all the inverse problems, we ask whether we can find a concrete family of based scoring schemes (scoring schemes obtained via a weighted scoring base) that have a positive solution to all the inverse problems?

\section{$(L^p,\lambda,\mu,A)$-based Scoring Schemes}\label{subsec:Lp lambda mu A}

In this section we construct a family of based scoring schemes which answers our four problems in the positive, and link this family back to the weighted max-based, h-categoriser and card-based semantics. In Section \ref{sec:new_semantics}, we show how this family of based scoring schemes can be used to define additional semantics which also answer the four inverse problems positively.

We fix
\begin{itemize}
\item $1 \leq p \leq \infty$,

\item $\mu \colon \RR^n \to (0,1]^n$ independent of supports,

\item $\lambda \colon \RR^n \to \End([0,\infty)^n)$  independent of supports,

\item $n \times n$ matrix $A=(a_{i,j})$ such that $A \geq 0$, i.e $a_{i,j} \geq 0$.
\end{itemize}
Recall that the $L^p$-norm on $\RR^n$, \gls{lpnom}, is defined by
\begin{equation}\label{eqn:Lp norm}
\|(x_1,\dots,x_n)\|_p = 
\left\{
\begin{array}{ll}
(x_1^p+\cdots x_n^p)^{1/p} & \text{if $1 \leq p< \infty$} 
\\
\max\{ x_1,\dots,x_n\} & \text{if $p=\infty$}
\end{array}
\right.
\end{equation}
Given vectors $x,y \in \RR^n$, \gls{xwedgey} is the element-wise product of vectors defined as
\begin{equation}\label{eqn:wedge vectors}
x \wedge y = (x_1y_1,\dots,x_ny_n).
\end{equation}

With this data and notation, define $\kappa \colon W \to K$ and $\varphi \colon W \to \Func(X,[0,\infty)^n)$ as follows.
\begin{eqnarray}
\label{eqn:Lp lambda mu A kappa phi}
&& \kappa(w)_i = \mu(w)_i  w_i 
\\
\nonumber
&& \varphi(w)(x)_i =  \| \lambda(w)(a_{i,*}) \wedge x\|_p. 
\end{eqnarray}

Note that in second line of Equation (\ref{eqn:Lp lambda mu A kappa phi}), \gls{ai*} denotes the $i^\text{th}$ row of $A$.

\begin{proposition}
\label{prop:kappa phi of Lp}
The pair $(\kappa,\varphi)$ defined in \eqref{eqn:Lp lambda mu A kappa phi} is a weighted scoring base.
\end{proposition}

\begin{defn}\label{defn:Lp lambda mu A scheme}
The associated scoring scheme $\sigma_{(\kappa,\varphi)}$ defined in \eqref{eqn:Lp lambda mu A kappa phi} is called the $(L^p,\lambda,\mu,A)$-based scoring scheme \gls{LplambdamuA}.
\end{defn}

\begin{theorem}\label{theorem:Lp-based solve everything}
Any $(L^p,\lambda,\mu,A)$-based scoring scheme has positive solution to the inverse problem, the reflection problem, the projective preference ordering problem and the radiality problem.

The discerning right inverse function $\inv \colon X \to \RR^n$ has the form
\[
\inv(x)_i = \tfrac{1}{\mu(x)_i}x_i(1+ \|\lambda(x)(a_{i,*}) \wedge x \|_p).
\]
If $\mu$ and $\lambda$ are constant functions then the topological reflection problem is answered positively.
\end{theorem}

The theorem above is of fundamental importance. Not only does it demonstrate that any $(L^p,\lambda,\mu,A)$-based scoring scheme has a solution to several of our fundamental problems, but it also describes how, and when it is possible, to compute initial weights from acceptability degrees.

\begin{theorem}\label{theorem:Lp subseteq Lq}
Suppose that $p \leq q$ then 
\[
D_{\sigma_{(L^p,\lambda,\mu,A)}} \subseteq D_{\sigma_{(L^q,\lambda,\mu,A)}}.
\]
\end{theorem}


An abstract weighted $(L^p,\lambda,\mu)$-based gradual semantics is a function which associates a $(L^p,\lambda,\mu,A)$-based scoring scheme to any non-weighted argumentation framework. The weighted gradual semantics $\Sigma_{\HC}, \Sigma_{\MB}$ and $\Sigma_{\CB}$ in Examples \ref{example:MB semantics}--\ref{example:HC semantics}, which are our prime source of interest are all instances of  $(L^p,\lambda,\mu)$-based gradual semantics.

\begin{proposition}\label{prop:HC MB CB are Lp rough version}
The weighted gradual semantics $\Sigma_{\HC}, \Sigma_{\MB}$ and $\Sigma_{\CB}$ are  \gls{awlplambdamuAbgs}.
\end{proposition}


\begin{example}[Example \ref{ex:scoringdynamics} cont.]

The gradual semantics $\Sigma_{\HC}$ is an abstract $(L^1, \lambda, \mu)$-based gradual semantics with:

\begin{align*}
& \mu(v) = \bar{1}  \qquad (v \in \RR^n)
\\
& \lambda(v) = \OP{id} \in \End([0,\infty)^n) \qquad (v \in \RR^n).
\end{align*}
where $\bar{1}$ denotes the vector $(1,\dots,1).$

Using Equation (\ref{eqn:Lp lambda mu A kappa phi}), the obtained pair $(\kappa, \varphi)$ is the exact weighted scoring base $(\kappa_1, \varphi_1)$ described in Example \ref{ex:scoringdynamics}.
\end{example}

As a result of Proposition \ref{prop:HC MB CB are Lp rough version}, we can deduce the following theorem about $\Sigma_{\HC},\Sigma_{\MB}$, and $\Sigma_{\CB}$.

\begin{theorem}\label{theorem:explicit MB CB HC}
The gradual semantics $\Sigma_{\HC},\Sigma_{\MB}$, and $\Sigma_{\CB}$ solve the inverse, reflection, projective preference ordering, and the radiality problems.
Moreover $\Sigma_{\HC},\Sigma_{\MB}$ solve the topological reflection problem.
\end{theorem}

We note that for $\Sigma_{\CB}$, $\mu$ and $\lambda$ are not constant functions and thus we cannot conclude that this semantics satisfies the topological reflection problem.

Using the discerning right inverse $\inv \colon X \to \RR^n$, we can easily compute initial weights from acceptability degrees. For the $\Sigma_{\HC},\Sigma_{\MB}$, and $\Sigma_{\CB}$ gradual semantics, these are given by the following formulae. We let $A$ denote the adjacency matrix an argumentation framework $\G=\langle \A,\D\rangle$, whose individual elements are denoted $a_{i,j}$.
Then 

\begin{align*}
\gls{invHC}(x)_i  &= x_i(1+\sum_j a_{i,j}x_j).
\\
\gls{invMB}(x)_i &= x_i(1+\max_j \{a_{i,j}x_j\})
\\
\gls{invCB}(x)_i &= x_i\Big(1+S(x)_i+\tfrac{1}{S(x)_i}(1+\sum_{j \in \supp(x)} a_{i,j}x_j)\Big).
\end{align*}
In the last formula $S(x)_i=\sum_{j \in \supp(x)} a_{i,j}$ and if $S(x)_i=0$ then by convention $\tfrac{1}{S(x)_i}=0$.

\begin{remark}
It is useful to identify $\wp(n)$ with $\{0,1\}^n$ as a subset of $\RR^n$, thus we view $\supp(x)$ as a vector of zeros and ones. 
If $v \in \RR^n$ we write $\diag(v)$ for the diagonal matrix with $\diag(v)_{ii}=v_i$.
In matrix form, one can write $\inv_{\HC},\inv_{\CB}$ and $\inv_{\MB}$ as follows.
Set
\[
S(x) = A \cdot \supp(x) 
\]
and write  $S(x)^{-1}:=(\tfrac{1}{S(x)_1},\dots,\tfrac{1}{S(x)_n})$ where $\tfrac{1}{S(x)_i}=0$ whenever $S(x)_i=0$. Also let $max_c(M)$ return a column-wise operation which retains the maximum element in a column, but sets all other entries in the column to 0 (breaking ties at random). 

Then
\begin{align*}
& \inv_{\HC}^A (x) = (I+\diag(x) \cdot A)\cdot x
\\
& \inv_{\CB}^A(x) = \Big(I+\diag(S(x))+\diag(S(x)^{-1}) \cdot A \cdot \diag(\supp(x))\Big) \cdot x \\
& \inv_{\MB}^A(x) = \Big(I+max_c(A \cdot diag(x))\Big) \cdot x
\end{align*}
\end{remark}

This result allows us to easily and precisely compute the initial weights for an argumentation framework and associated acceptability degrees, without resorting to numerical methods as was done in \cite{oren2022inverse}.

In this section we demonstrated that there exists a family of gradual semantics, namely the abtract weighted $(L^p, \lambda, \mu)$-based gradual semantics which satisfy the inverse, reflection, projective preference ordering, radiality problem, and, with additional constraints, also satisfy the topological reflection problem. The three weighted gradual semantics often considered in the literature are instantiations of this family of gradual semantics. Finally, thanks to the presence of the discerning right inverse in our gradual semantics family, it is easy to calculate the initial weights for an argumentation system given the acceptability degrees for each argument.

\section{New Gradual Semantics}
\label{sec:new_semantics}

With the machinery we have developed, there is a large collection of desirable weighted gradual semantics that we can define inspired by the content of Proposition \ref{prop:HC MB CB are Lp rough version}. In this section we provide two new example semantics built using our work.

\subsection{The Weighted Euclidean-based gradual semantics}

The weighted Euclidean-based gradual semantics \gls{SigmaEB} is the abstract weighted $(L^p,\lambda,\mu)$-based gradual semantics with $p=2$ and $\mu = 1$ and $\lambda$ constant with value $\OP{id} \in \End([0,\infty)^n)$.
Thus, $\Sigma_{\EB}^{\G}(w)$ is obtained as the limit of the sequence of vectors $\gls{OPEBk} \in X$ where $\OP{EB}_0$ is chosen arbitrarily and 
\[
\OP{EB}_{k+1}(x) = \frac{w_i}{1+\sqrt{\sum_{j \in \Att(i)}x_j^2}}.
\]

From our results, this semantics positively answers the inverse problem, topological reflection problem, the projective preference order problem and the radiality problem.
The discerning right inverse function $\inv_{\EB} \colon X \to \RR^n$ is given by
\[
\gls{invEB}(x)_i = x_i\left(1+\sqrt{\sum_{j \in \Att(i)}x_j^2} \right).
\]

As one can observe, it is easy to create basic variations of existing weighted gradual semantics while preserving important properties such as convergence or satisfactions to inverse problems.
This semantics is a specific instantiation of the parameterized compensation semantics introduced in \cite{DBLP:conf/ijcai/DoderAV23}. It should be noted that every such parameterized compensation semantics for which $\alpha \geq 1$ is in fact an abstract weighted $(L^p,\lambda,\mu)$-based gradual semantics.

\subsection{Gradual semantics with remote attacks}

In most weighted gradual semantics considered to date (including $\Sigma_\HC,$ $\Sigma_\MB$, and $\Sigma_\CB$), only the acceptability degree of direct attackers on arguments in the graph $\G$ were considered. 
The intuition given by this is that the acceptability degrees of direct attackers already ``compile'' the acceptability degree of indirect attackers. However, there may be cases where it is useful to access the degree of these indirect attackers.

That is, given $i \in \A$, attacks on $i$ are considered only by arguments $j \in \A$ such that $(j,i) \in \D$, i.e., $j$ is a neighbour at distance $1$ from $i$ within the underlying argumentation framework.
The ``amount'' of attack on $i$ is formulated by means of a function $\alpha_i$ which only depends on $x_j$ with $j \in \Att(i)$.
For example, in $\Sigma_{\HC}$ we get $\alpha_i(x) = \sum_{j \in \Att(j)} x_j$ and in $\Sigma_{\MB}$ it is $\alpha_i(x)=\max_{j \in \Att(i)} x_j$. More generally, if $\Sigma$ is associated to a weighted scoring base $(\kappa,\varphi)$ then the attack is formulated by means of the functions $\varphi_i \colon X \to [0,\infty)$.

Instead, it is conceivable to consider ``secondary level'' attacks on $i$ by arguments $j$ at distance greater than $1$. Such attacks are loosely analogous to indirect defeaters in bipolar argumentation frameworks \cite{amgoud_bipolarity_2008}, but will have ``weaker'' strength which is achieved by multiplying the function $\varphi$ by a discount factor $\gls{delta}<1$.

Recall that if $A$ is the adjacency matrix of $A$ then the $(i,j)^{\text{th}}$ entry of $A^k$ is equal to the number of (directed) paths of length $k$ from $j$ to $i$, thus the number of ``secondary attacks'' of $j$ on $i$ via a sequence of $k+1$ arguments (including $i$ and $j$).
In order to encode such ``remote attacks'' one should consider the matrix
\[
A' = A+\delta_2A^2+\delta_3A^3 + \cdots
\]
where $0 \leq \delta_k \leq 1$ indicate the multiplicative strength factors of attacks from distance $k$.

As a concrete example, let us show how we modify $\Sigma_{\HC}$ to a new weighted gradual semantics $\Sigma_{\HC,r}$ which takes into account remote attacks from distance $2$.
Choose some $0<\delta<1$.
Suppose that $\G$ is an argumentation framework and let $A$ be its adjacency matrix.
Set
\[
A' = A+\delta  A^2.
\]
Consider the abstract  $(L^p,\lambda,\mu,A')$-based scoring scheme with $p=1$ and $\lambda=\mu=1$.
Then attacks on $i$ are made using the formula
\[
\frac{w_i}{1+(A'x)_i} =
\frac{w_i}{1+(Ax)_i+\delta  (A^2x)_i} 
\]
and observe that $Ax_i=\sum_{j \in \Att(i)}x_j$ and $(A^2x)_i = \sum_{\gamma \in \Att_2(i)} x_j$ where $\gls{Att2i}$ is the set of all paths $\gls{gamma}$ of length $2$ in $\G$ ending at $i$, and the term $x_j$ in the sum refers to the argument starting the path $\gamma$.
We highlight that here, the degree of indirect attackers may be counted multiple times if there are multiple paths.

To generalise this result to remote attacks further away, it is reasonable to use the matrix
\[
A' = \sum_{k=1}^\infty \delta^{k-1}A^k.
\]

This second example shows the richness of the family of abstract weighted $(L^p,\lambda,\mu)$-based gradual semantics.

\section{Conclusions}
\label{sec:conclusion}

We began this paper by rephrasing the notion of a weighted gradual semantics in terms of scoring schemes. This allowed us to formulate four problems, including the inverse problem of \cite{oren2022inverse} around such scoring schemes. These four problems revolve around the uniqueness of solutions, and continuity of solutions, for weighted gradual semantics, and thus encode highly desirable properties of any semantics.

We introduced a very general class of gradual semantics, the abstract weighted $(L^p,\lambda,\mu)-$based gradual semantics. Any semantics in this class (including the weighted h-categoriser, max-based and card-based semantics) answers the four problems we describe in the positive.  We also demonstrated that an analytical solution to the original inverse problem can be derived, in contrast to the numeric approach used in \cite{oren2022inverse}.

Building on the generality of our semantics and the desirable properties of our problems, we described two new  semantics (the weighted Euclidean-based gradual semantics and gradual semantics with remote attacks).

Our paper addresses several important theoretical gaps in argumentation theory, and facilitates computationally efficient solutions to the inverse problem. In the context of future work, the new semantics we propose in Section \ref{sec:new_semantics} are very interesting. One obvious avenue of such future work involves determining which of the properties described in \cite{AMGOUD2022103607} this family of new semantics comply with. 
Amgoud and Doder \cite{DBLP:conf/atal/AmgoudD19} described a general class of gradual semantics which contains evaluation method-based gradual semantics. These were shown to converge, and satisfy previously considered desirable properties (e.g., directionality, equivalence and weakening among others). While our semantics also converge (c.f., Theorem \ref{theorem:T_c f in Base(X)}), our focus in this paper is on the four inverse problems. Thus, relating our semantics with the class of evaluation method-based gradual semantics and identifying where they overlap and where they diverge remains future work.
 In addition, the $\sigma$-acceptability-maximal semantics and ratio semantics differ significantly from existing semantics, and offer the possibility of making elicitation from expert knowledge easier, reducing knowledge engineering effort. Thus, exploring the properties of these semantics, and how closely they mirror human reasoning, is an avenue of work we intend to pursue.

\bibliographystyle{unsrt}  
\bibliography{biblio}

\begin{thebibliography}{10}

\bibitem{dung_acceptability_1995}
Phan~Minh Dung.
\newblock On the {Acceptability} of {Arguments} and its {Fundamental} {Role} in
  {Nonmonotonic} {Reasoning}, {Logic} {Programming} and n-{Person} {Games}.
\newblock {\em Artif. Intell.}, 77(2):321--358, 1995.

\bibitem{baroni_introduction_2011}
Pietro Baroni, Martin Caminada, and Massimiliano Giacomin.
\newblock An introduction to argumentation semantics.
\newblock {\em Knowledge Eng. Review}, 26(4):365--410, 2011.

\bibitem{verheij1996two}
Bart Verheij.
\newblock Two approaches to dialectical argumentation: admissible sets and
  argumentation stages.
\newblock {\em Proc. NAIC}, 96:357--368, 1996.

\bibitem{DBLP:conf/comma/Caminada06}
Martin Caminada.
\newblock Semi-stable semantics.
\newblock In Paul~E. Dunne and Trevor J.~M. Bench{-}Capon, editors, {\em
  Computational Models of Argument: Proceedings of {COMMA} 2006, September
  11-12, 2006, Liverpool, {UK}}, volume 144 of {\em Frontiers in Artificial
  Intelligence and Applications}, pages 121--130. {IOS} Press, 2006.

\bibitem{DBLP:journals/ai/DungMT07}
Phan~Minh Dung, Paolo Mancarella, and Francesca Toni.
\newblock Computing ideal sceptical argumentation.
\newblock {\em Artif. Intell.}, 171(10-15):642--674, 2007.

\bibitem{amgoud_ranking-based_2013}
Leila Amgoud and Jonathan Ben-Naim.
\newblock Ranking-{Based} {Semantics} for {Argumentation} {Frameworks}.
\newblock In {\em Proc.~Scalable {Uncertainty} {Management}}, pages 134--147,
  2013.

\bibitem{bonzon_comparative_2016}
Elise Bonzon, J{\'e}r{\^o}me Delobelle, S{\'e}bastien Konieczny, and Nicolas
  Maudet.
\newblock A {Comparative} {Study} of {Ranking}-{Based} {Semantics} for
  {Abstract} {Argumentation}.
\newblock In {\em Proc.~{AAAI}}, pages 914--920, 2016.

\bibitem{DBLP:journals/jancl/BonzonDKM23}
Elise Bonzon, J{\'{e}}r{\^{o}}me Delobelle, S{\'{e}}bastien Konieczny, and
  Nicolas Maudet.
\newblock An empirical and axiomatic comparison of ranking-based semantics for
  abstract argumentation.
\newblock {\em J. Appl. Non Class. Logics}, 33(3-4):328--386, 2023.

\bibitem{yun20ranking}
Bruno Yun, Srdjan Vesic, and Madalina Croitoru.
\newblock Ranking-based semantics for sets of attacking arguments.
\newblock In {\em Proceedings of the AAAI Conference on Artificial
  Intelligence}, pages 3033--3040, June 2020.

\bibitem{DBLP:conf/kr/AmgoudD18a}
Leila Amgoud and Dragan Doder.
\newblock Gradual semantics for weighted graphs: An unifying approach.
\newblock In Michael Thielscher, Francesca Toni, and Frank Wolter, editors,
  {\em Principles of Knowledge Representation and Reasoning: Proceedings of the
  Sixteenth International Conference, {KR} 2018, Tempe, Arizona, 30 October - 2
  November 2018}, pages 613--614. {AAAI} Press, 2018.

\bibitem{yun_2021_gradual}
Bruno Yun and Srdjan Vesic.
\newblock Gradual semantics for weighted bipolar setafs.
\newblock In {\em {ECSQARU} 2021}, volume 12897 of {\em Lecture Notes in
  Computer Science}, pages 201--214. Springer, 2021.

\bibitem{besnard_logic-based_2001}
Philippe Besnard and Anthony Hunter.
\newblock A logic-based theory of deductive arguments.
\newblock {\em Artif. Intell.}, 128(1-2):203--235, 2001.

\bibitem{AMGOUD2022103607}
Leila Amgoud, Dragan Doder, and Srdjan Vesic.
\newblock Evaluation of argument strength in attack graphs: Foundations and
  semantics.
\newblock {\em Artificial Intelligence}, 302:103607, 2022.

\bibitem{oren2022inverse}
Nir Oren, Bruno Yun, Srdjan Vesic, and Murilo~S. Baptista.
\newblock Inverse problems for gradual semantics.
\newblock In Luc~De Raedt, editor, {\em Proceedings of the Thirty-First
  International Joint Conference on Artificial Intelligence, {IJCAI} 2022,
  Vienna, Austria}, pages 2719--2725, 2022.

\bibitem{pu_argument_2014}
Fuan Pu, Jian Luo, Yulai Zhang, and Guiming Luo.
\newblock Argument {Ranking} with {Categoriser} {Function}.
\newblock {\em CoRR}, abs/1406.3877, 2014.

\bibitem{DBLP:conf/comma/SkibaTRHK22}
Kenneth Skiba, Matthias Thimm, Tjitze Rienstra, Jesse Heyninck, and Gabriele
  Kern{-}Isberner.
\newblock Realisability of rankings-based semantics.
\newblock In Sarah~Alice Gaggl, Jean{-}Guy Mailly, Matthias Thimm, and
  Johannes~Peter Wallner, editors, {\em Proceedings of the Fourth International
  Workshop on Systems and Algorithms for Formal Argumentation co-located with
  the 9th International Conference on Computational Models of Argument {(COMMA}
  2022), Cardiff, Wales, United Kingdom, September 13, 2022}, volume 3236 of
  {\em {CEUR} Workshop Proceedings}, pages 73--85. CEUR-WS.org, 2022.

\bibitem{10.5555/3171642.3171679}
Hiroyuki Kido and Keishi Okamoto.
\newblock A bayesian approach to argument-based reasoning for attack
  estimation.
\newblock In {\em Proceedings of the 26th International Joint Conference on
  Artificial Intelligence}, IJCAI'17, page 249–255. AAAI Press, 2017.

\bibitem{DBLP:journals/jancl/KidoL22}
Hiroyuki Kido and Beishui Liao.
\newblock A bayesian approach to forward and inverse abstract argumentation
  problems.
\newblock {\em J. Appl. Non Class. Logics}, 32(4):273--304, 2022.

\bibitem{DBLP:journals/corr/abs-2211-16118}
Nir Oren and Bruno Yun.
\newblock Inferring attack relations for gradual semantics.
\newblock {\em Argument \& Computation}, 14(3):327--345, 2023.

\bibitem{DBLP:conf/clar/Mailly23}
Jean{-}Guy Mailly.
\newblock A logical encoding for k-m-realization of extensions in abstract
  argumentation.
\newblock In Andreas Herzig, Jieting Luo, and Pere Pardo, editors, {\em Logic
  and Argumentation - 5th International Conference, {CLAR} 2023, Hangzhou,
  China, September 10-12, 2023, Proceedings}, volume 14156 of {\em Lecture
  Notes in Computer Science}, pages 84--100. Springer, 2023.

\bibitem{coste-marquis_selecting_2012}
Sylvie Coste-Marquis, S{\'e}bastien Konieczny, Pierre Marquis, and Mohand~Akli
  Ouali.
\newblock Selecting {Extensions} in {Weighted} {Argumentation} {Frameworks}.
\newblock In {\em Computational {Models} of {Argument} - {Proceedings} of
  {COMMA} 2012, {Vienna}, {Austria}, {September} 10-12, 2012}, pages 342--349,
  2012.

\bibitem{coste-marquis_weighted_2012}
Sylvie Coste{-}Marquis, S{\'{e}}bastien Konieczny, Pierre Marquis, and
  Mohand~Akli Ouali.
\newblock Weighted attacks in argumentation frameworks.
\newblock In Gerhard Brewka, Thomas Eiter, and Sheila~A. McIlraith, editors,
  {\em Principles of Knowledge Representation and Reasoning: Proceedings of the
  Thirteenth International Conference, {KR} 2012, Rome, Italy, June 10-14,
  2012}, pages 593--597. {AAAI} Press, 2012.

\bibitem{dunne_weighted_2011}
Paul~E. Dunne, Anthony Hunter, Peter McBurney, Simon Parsons, and Michael
  Wooldridge.
\newblock Weighted argument systems: {Basic} definitions, algorithms, and
  complexity results.
\newblock {\em Artif. Intell.}, 175(2):457--486, 2011.

\bibitem{DBLP:conf/prima/MahesarOV20}
Quratul{-}ain Mahesar, Nir Oren, and Wamberto~Weber Vasconcelos.
\newblock Preference elicitation in assumption-based argumentation.
\newblock In Takahiro Uchiya, Quan Bai, and Ivan Mars{\'{a}}{-}Maestre,
  editors, {\em {PRIMA} 2020: Principles and Practice of Multi-Agent Systems -
  23rd International Conference, Nagoya, Japan, November 18-20, 2020,
  Proceedings}, volume 12568 of {\em Lecture Notes in Computer Science}, pages
  199--214. Springer, 2020.

\bibitem{DBLP:conf/ijcai/Coste-MarquisKM15}
Sylvie Coste{-}Marquis, S{\'{e}}bastien Konieczny, Jean{-}Guy Mailly, and
  Pierre Marquis.
\newblock Extension enforcement in abstract argumentation as an optimization
  problem.
\newblock In Qiang Yang and Michael~J. Wooldridge, editors, {\em Proceedings of
  the Twenty-Fourth International Joint Conference on Artificial Intelligence,
  {IJCAI} 2015, Buenos Aires, Argentina, July 25-31, 2015}, pages 2876--2882.
  {AAAI} Press, 2015.

\bibitem{DBLP:conf/ecsqaru/DelobelleV19}
J{\'{e}}r{\^{o}}me Delobelle and Serena Villata.
\newblock Interpretability of gradual semantics in abstract argumentation.
\newblock In Gabriele Kern{-}Isberner and Zoran Ognjanovic, editors, {\em
  Symbolic and Quantitative Approaches to Reasoning with Uncertainty, 15th
  European Conference, {ECSQARU} 2019, Belgrade, Serbia, September 18-20, 2019,
  Proceedings}, volume 11726 of {\em Lecture Notes in Computer Science}, pages
  27--38. Springer, 2019.

\bibitem{DBLP:journals/corr/abs-2401-08879}
Timotheus Kampik, Nico Potyka, Xiang Yin, Kristijonas Cyras, and Francesca
  Toni.
\newblock Contribution functions for quantitative bipolar argumentation graphs:
  {A} principle-based analysis.
\newblock {\em CoRR}, abs/2401.08879, 2024.

\bibitem{jeanguy13dynamic}
Jean-Guy Mailly.
\newblock Dynamics of argumentation frameworks.
\newblock In {\em Proceedings of the Twenty-Third International Joint
  Conference on Artificial Intelligence}, IJCAI '13, page 3233–3234. AAAI
  Press, 2013.

\bibitem{tarle22multiagent}
Louise Dupuis~de Tarl\'{e}, Elise Bonzon, and Nicolas Maudet.
\newblock Multiagent dynamics of gradual argumentation semantics.
\newblock In {\em Proceedings of the 21st International Conference on
  Autonomous Agents and Multiagent Systems}, AAMAS '22, page 363–371,
  Richland, SC, 2022. International Foundation for Autonomous Agents and
  Multiagent Systems.

\bibitem{skitalinskaya-etal-2023-claim}
Gabriella Skitalinskaya, Maximilian Splieth{\"o}ver, and Henning Wachsmuth.
\newblock Claim optimization in computational argumentation.
\newblock In C.~Maria Keet, Hung-Yi Lee, and Sina Zarrie{\ss}, editors, {\em
  Proceedings of the 16th International Natural Language Generation
  Conference}, pages 134--152, Prague, Czechia, September 2023. Association for
  Computational Linguistics.

\bibitem{Pu14}
Fuan Pu, Jian Luo, Yulai Zhang, and Guiming Luo.
\newblock Argument ranking with categoriser function.
\newblock In Robert Buchmann, Claudiu~Vasile Kifor, and Jian Yu, editors, {\em
  Knowledge Science, Engineering and Management - 7th International Conference,
  {KSEM} 2014, Sibiu, Romania, October 16-18, 2014. Proceedings}, volume 8793
  of {\em Lecture Notes in Computer Science}, pages 290--301. Springer, 2014.

\bibitem{DBLP:conf/ijcai/DoderAV23}
Dragan Doder, Leila Amgoud, and Srdjan Vesic.
\newblock Parametrized gradual semantics dealing with varied degrees of
  compensation.
\newblock In {\em Proceedings of the Thirty-Second International Joint
  Conference on Artificial Intelligence, {IJCAI} 2023, 19th-25th August 2023,
  Macao, SAR, China}, pages 3176--3183. ijcai.org, 2023.

\bibitem{amgoud_bipolarity_2008}
Leila Amgoud, Claudette Cayrol, Marie-Christine Lagasquie-Schiex, and P.~Livet.
\newblock On bipolarity in argumentation frameworks.
\newblock {\em Int. J. Intell. Syst.}, 23(10):1062--1093, 2008.

\bibitem{DBLP:conf/atal/AmgoudD19}
Leila Amgoud and Dragan Doder.
\newblock Gradual semantics accounting for varied-strength attacks.
\newblock In Edith Elkind, Manuela Veloso, Noa Agmon, and Matthew~E. Taylor,
  editors, {\em Proceedings of the 18th International Conference on Autonomous
  Agents and MultiAgent Systems, {AAMAS} '19, Montreal, QC, Canada, May 13-17,
  2019}, pages 1270--1278. International Foundation for Autonomous Agents and
  Multiagent Systems, 2019.

\end{thebibliography}

\newpage

 \appendix

\section{}
\label{sec:symbols}

{\small
\printglossary[title=List of symbols and terms,style=mcolindex]
}

\newpage

\section{}

This appendix contains the proofs of the results presented in this paper. We kept the order in which we prove our results in the same order in which we present them.
All definitions, lemmas, theorems, and propositions in this appendix are intermediary results which are used in the proofs.

We will use the following notation: Given $a,b \in X$ such that $a \preceq b$, the {\bf interval} between $a$ and $b$ in $X$ is $[a,b]_X \overset{\text{def}}{=} \{ x \in X \, : \, a \preceq x \preceq b\}. $

\begin{proof}[\textbf{Proof of Theorem \ref{th:pu-generalisation}}] Define a sequence $u^{(n)}$ in $X$ by recursion by $u^{(0)}=0$ and $u^{(n+1)}=f(u^{(n)})$ for every $n \geq 0$.
Then $u^{(0)} \preceq u^{(1)}$ since $0 \in X$ is the minimum element in $(X,\preceq)$.
Since $f$ is order reversing, we get $u^{(0)} \preceq u^{(2)} \preceq u^{(1)}$.
By induction one easily shows
\begin{enumerate}[label=(\roman*)]
\item\label{Th:pu:property i}
$u^{(2k)} \preceq u^{(2k-1)}$ \ for all $k \geq 1$,

\item\label{Th:pu:property ii}
$u^{(2k)} \preceq u^{(2k+2)}$ \  for all $k \geq 0$, 

\item\label{Th:pu:property iii}
$u^{(2k+1)} \preceq u^{(2k-1)}$ \ for all $k \geq 1$.
\end{enumerate}
%
%
Thus, the sequence $u^{(2k)}$ is increasing in the sense that $u^{(2k)}_i$ is an increasing sequence in $\RR$ for every $i=1,\dots,n$, and bounded above by $1 \in X$.
Similarly, $u^{(2k-1)}$ is decreasing and bounded below by $0 \in X$.
We may therefore define 
\begin{align*}
    u^{(\OP{ev})} &= \sup_k u^{(2k)} = \lim_k u^{(2k)} \\ 
    u^{(\OP{odd})} &= \inf_k u^{(2k-1)} = \lim_k u^{(2k-1)}.
\end{align*}
It follows from \ref{Th:pu:property i} that $u^{(\OP{ev})} \preceq u^{(\OP{odd})}$.

Our next goal is to show that $u^{(\OP{ev})} = u^{(\OP{odd})}$.
For every $k \geq 1$, set:
\[
\pi_k = \sup \ \{ 0 \leq t \leq 1 \ :\  t \cdot u^{(2k-1)} \preceq u^{(2k)} \}.
\]
The set on the right is not empty since $0 \cdot u^{(2k-1)}=0 \preceq u^{(2k)}$, so:
\[
0 \leq \pi_k \leq 1.
\]
It is clear from the construction that for every $k \geq 1$
\[
\pi_k \cdot u^{(2k-1)} \preceq u^{(2k)}.
\]
By applying hypotheses \ref{th:pu:reversing} and \ref{th:pu:bound} to this inequality and then using \ref{Th:pu:property ii}, for every $k \geq 1$
\begin{align*}
u^{(2k+1)} &= f(u^{(2k)}) 
\\
& \preceq f(\pi_k \cdot u^{(2k-1)}) 
\\
& \preceq \tfrac{1}{\pi_k+\alpha(1-\pi_k)} \cdot f(u^{(2k-1)}) 
\\
&= \tfrac{1}{\pi_k+\alpha(1-\pi_k)} \cdot u^{(2k)} \\
& \preceq \tfrac{1}{\pi_k+\alpha(1-\pi_k)} \cdot u^{(2k+2)}.
\end{align*}
Hence $(\pi_k+\alpha(1-\pi_k)) \cdot u^{(2k+1)} \preceq u^{(2k+2)}$, so by the definition of $\pi_{k+1}$
\[
\pi_k+\alpha(1-\pi_k) \leq \pi_{k+1}.
\]
It follows that $1-\pi_{k+1} \leq (1-\alpha)(1-\pi_k)$ for all $k \geq 1$.
Therefore 
\[
1-\pi_{k+1} \leq (1-\pi_1) (1-\alpha)^k \xto{ \ k \to \infty \ } 0.
\]
But $\pi_k \leq 1$ for all $k$, so $\lim_k \pi_k =1$.

Consider some $\epsilon>0$.
Then $\pi_k >1-\epsilon$ for all $k \gg 0$.
Equation \ref{Th:pu:property i} implies
\[
(1-\epsilon) \cdot u^{(2k-1)} \preceq \pi_k \cdot u^{(2k-1)} \preceq  u^{(2k)} \preceq u^{(2k-1)}.
\]
Letting $k \to \infty$, this implies $(1-\epsilon) u^{(\OP{odd})} \preceq u^{(\OP{ev})} \preceq u^{(\OP{odd)}}$.
Since $\epsilon>0$ was arbitrary,
\[
u^{(\OP{ev)}}=u^{(\OP{odd})}
\]
as needed.
We will denote $u^*:=u^{(\OP{ev)}}=u^{(\OP{odd})}$.

Write $f^n \colon X \to X$ for the $n$-fold composition of $f$ with itself, i.e \,  $f^n=f \circ \cdots \circ f$. 
Since $0 \in X$ is minimal, $u^{(0)} \preceq x$ for any $x \in X$.
Since $f$ is order reversing, $u^{(0)} \preceq f(x) \preceq u^{(1)}$.
It then follows by induction that for all $k \geq 1$
\begin{eqnarray*}
&& f^{2k-1}(X) \subseteq [u^{(2k-2)},u^{(2k-1)}]_X \\
&& f^{2k}(X) \subseteq [u^{(2k)},u^{(2k-1)}]_X.
\end{eqnarray*}

It follows from Equation \ref{Th:pu:property ii} that
\begin{align*}
\bigcap_{n \geq 1} f^n(X) &= 
\bigcap_{k \geq 1} f^{2k-1}(X) \cap f^{2k}(X) 
\\
& \subseteq 
\bigcap_{k \geq 1} [u^{(2k-2)},u^{(2k-1)}]_X \cap [u^{(2k)},u^{(2k-1)}]_X 
\\
& = \bigcap_{k \geq 1} [u^{(2k)},u^{(2k-1)}]_X 
\\
  &= [u^{(\OP{ev})},u^{(\OP{odd})}]_X = \{u^*\}.
\end{align*}
Notice that $u^* \in [u^{(2k)},u^{(2k-1)}]_X$ 
for all $k \geq 1$. 
Since $f$ is order reversing $f(u^*) \in [u^{(2k)},u^{(2k+1)}]_X$ 
for all $k \geq 1$.
Therefore $f(u^*) \in \{u^*\}$, so $u^*$ is a fixed point of $f$.

If $y \in X$ is a fixed point of $f$ then by induction $y \in f^n(X)$ for all $n$, so $y \in \cap_n f^n(X)=\{u^*\}$ and it follows that $y=u^*$.
Therefore $u^*$ is the unique fixed point of $f$.

Finally, choose some $x \in X$ and define a sequence by recursion $x^{(0)}=x$ and $x^{(n+1)}=f(x^{(n)})$.
Then $u^{(0)}=0 \preceq x$ and since $f$ is order reversing, $u^{(0)} \preceq x^{(1)} \preceq u^{(1)}$.
Then one proves by induction that $u^{(2k-2)} \preceq x^{(2k-1)} \preceq u^{(2k-1)}$ and $u^{(2k)} \preceq x^{(2k)} \preceq u^{(2k-1)}$ for all $k \geq 1$.
By the sandwich rule $\lim_n x^{(n)}=u^*$.
\end{proof}

\begin{proof}[\textbf{Proof of Theorem \ref{theorem:T_c f in Base(X)}}]
Let $(c,f) \in \Base(X)$.
Thus $c \in K$ and $f \colon X \to [0,\infty)^n$ is bounded, homogeneous and increasing.
Set $F=T_{(c,f)}$, see Definition \ref{defn:T_kappa phi}.
Recall that $F \in \End(X)$.

\medskip
\noindent
{\em Claim 1:} $F$ is order reversing.

\noindent
{\em Proof of Claim 1:}
Suppose that $x \preceq y$, i.e $x_i \leq y_i$ for all $i=1,\dots, n$.
Since $f_i \colon X \to [0,\infty)$ is increasing
\[
F(y)_i = \frac{c_i}{1+f(y)_i} \leq \frac{c_i}{1+f(x)_i} =F(x)_i.
\]
It follows that $F(y) \preceq F(x)$.
This proves Claim 1.

\medskip
\noindent
{\em Claim 2:} 
There exists $0 < \alpha \leq 1$ such that for any $x \in X$ and any $0 \leq t \leq 1$
\[
F(tx) \preceq \frac{1}{t+\alpha\cdot(1-t)} \cdot F(x).
\]
\noindent
{\em Proof of Claim 2:}
Say that the functions $f_i \colon X \to [0,\infty)$ are bounded above by $M>0$.
Set $\alpha =\tfrac{1}{1+M}$.
Then for any $x \in X$ 
\[
\frac{1}{1+f(x)_i} \geq \frac{1}{1+M} = \alpha.
\]
Since $f_i$ is homogeneous, for any $x \in X$ and any $0 \leq t \leq 1$,
\begin{align*}
F(t \cdot x)_i &= \frac{c_i}{1+f_i(t \cdot x)}
\\
&= \frac{c_i}{1+t\cdot f_i(x)} \\
&= \frac{c_i}{(1-t)+t(1+f_i(x))} \\
& = \frac{c_i}{1+f_i(x)} \cdot \frac{1+f_i(x)}{(1-t)+t(1+f_i(x))} \\
&= F(x)_i \cdot \frac{1}{(1-t)\tfrac{1}{1+f_i(x)}+t} \\
& \leq \frac{1}{(1-t)\alpha+t} \cdot F(x)_i.
\end{align*}
%
The theorem follows by applying Theorem \ref{th:pu-generalisation} to $F$.
\end{proof}

\begin{defn}\label{def:independent of weights}
Let $Y$ be a set.
Let $U$ be a subset of $W$.
A function $f \colon W \to Y$ is called {\bf independent of weights} on $U$ if $f|_U \colon U \to Y$ is a constant function.
The constant value of $f|_U$, i.e the element $y \in Y$ such that $f|_U \equiv y$,  is called the {\bf $U$-value} of $f$.

We say that $f$ is independent of weights if it is independent of weights on $W$, i.e it is constant.
\end{defn}

Recall that a topological space $X$ is called {\em locally compact} if every $x \in X$ has a compact neighbourhood, i.e a neighbourhood whose closure is compact.

\begin{lemma}\label{lem:homeomorphism criterion}
Let $X,Y$ be metric spaces, $X$ compact.
Let $D \subseteq X$ be a subspace and $C \subseteq Y$ locally compact subspace.
Let $k \colon X \to Y$ be  a continuous function such that
\begin{enumerate}[label=(\alph*)]
\item\label{lem:homeo:inverse}
$k(D)=C$ and the restriction $k|_D \colon D \to C$ is a bijective function with inverse $f \colon C \to D$.

\item\label{lem:homeo:reflecting}
For every $x \in X$,
\[
k(x) \in C \iff x \in D.
\]
\end{enumerate}
Then $k|_D$ and $f$ are homeomorphisms.
\end{lemma}

\begin{proof}
Since $k|_D$ is continuous, it remains to show that its inverse $f$ is continuous.
To do this it suffices to show that $f$ is continuous at any $c \in C$. 

Let $c \in C$.
Since $C$ is locally compact, choose a neighbourhood $c \in U \subseteq C$ whose closure $\overline{U}$ in $C$ is compact.
Set 
\[
E:=f(\overline{U}) \subseteq D.
\]
We will next show that $E$ is a closed subset of $X$.
Let $(e_n)$ be a convergent sequence in $X$ contained in $E$ and set $x=\lim_{n \to \infty} e_n$.
We must show that $x \in E$.

By construction $e_n=f(u_n)$ for some $u_n \in \overline{U}$.
Since $k$ is continuous, the sequence $k(e_n)$ is a convergent sequence in $Y$ and 
\[
\lim_{n\to \infty} k(e_n)=k(x).
\]
By the assumption \ref{lem:homeo:reflecting} $k(e_n) \in C$ (because $e_n \in E \subseteq D$), and assumption \ref{lem:homeo:inverse} implies that
\[
f(k(e_n))=e_n=f(u_n).
\]
Since $f$ is injective, $k(e_n)=u_n \in \overline{U}$.
Since $\overline{U}$ is compact it is closed in $Y$, hence $k(x) \in \overline{U} \subseteq C$.
Assumption \ref{lem:homeo:reflecting} implies that $x \in D$, and by assumption \ref{lem:homeo:inverse}
\[
x=f(k(x)) \in f(\overline{U})=E.
\]
This completes the proof that $E$ is a closed subset of $X$.

Since $X$ is compact, so is $E$.
Assumption \ref{lem:homeo:inverse} now implies that $k|_E \colon E \to \overline{U}$ is a bijective continuous function between compact metric spaces, hence it is a homeomorphism.
Therefore its inverse $f|_{\overline{U}}$ is continuous and in particular, since $U$ is open, $f$ is continuous at $c$.
This completes the proof.
\end{proof}

\begin{theorem}\label{th:F bar factorisation}
Let $b \colon W \to \Base(X)$ be a  weighted scoring base.
Write $b=(\kappa,\varphi)$ as in Definition \ref{def:weighted scoring base}.
Assume that $\varphi$ is independent of weights on some $U \subseteq W$ with $U$-value $\psi \colon X \to [0,\infty)^n$.
Then
\begin{enumerate}[label=(\alph*)]
\item\label{th:Fbar:exists}
There exists a bijective function 
\[
\overline{\sigma_{b,U}} \colon \kappa(U) \xto{\ \cong \ } \sigma_{b}(U) \subseteq X
\]
which renders the following triangle commutative
\[
\xymatrix{
U \ar[r]^{\sigma_b} \ar@{->>}[d]_{\kappa} &
X 
\\
\kappa(U) \ar[ur]_{\overline{\sigma_{b,U}}}
}
\]

\item\label{th:Fbar:inverse}
The inverse of $\overline{\sigma_{b,U}}$ is given by the restriction to $\sigma_{b}(U) \subseteq X$ of the function
\[
\overline{\inv_{b,U}} \colon X \to \RR^n
\]
defined on each components by
\[
\overline{\inv_{b,U}}(x)_i = x_i(1+\psi(x)_i).
\]

\item\label{th:Fbar:reflect}
For any $x \in X$,
\[
\overline{\inv_{b,U}}(x) \in \kappa(U) \iff x \in \sigma_b(U).
\]

\item\label{th:Fbar:homeomorphism}
If $\kappa(U)$ is locally compact then $\overline{\sigma_{b,U}}$ is a homeomorphism.
\end{enumerate} 
\end{theorem}

\begin{proof}
\ref{th:Fbar:exists} and \ref{th:Fbar:inverse}.
Define $\overline{\sigma_{b,U}} \colon \kappa(U) \to \sigma_b(U)$ as follows.
Given $c \in \kappa(U)$ choose $u \in U$ such that $\kappa(u)=c$ and set
\[
\overline{\sigma_{b,U}}(c) = \sigma_b(u)
\]
We need to show that this is independent of the choices.
If $u' \in U$ is another preimage of $c$ then by the definition of $\sigma_b$ in Definition \ref{def:weighted scoring base} and since $\varphi$ is independent of weights on $U$ and  $\kappa(u)=c=\kappa(u')$
\[
\sigma_b(u') = \fix(T_{(\kappa(u'),\varphi(u'))}) = \fix(T_{(\kappa(u),\varphi(u)})) = \sigma_b(u).
\]
Hence $\overline{\sigma_{b,U}}(c)$ is independent of the choice of the preimages of $c$ in $U$.
By construction, $\overline{\sigma_{b,U}}(c) \in \sigma_b(U)$ and the diagram is commutative, that is
\[
\overline{\sigma_{b,U}} \circ \kappa|_U = \sigma_b|_U.
\]
In particular $\overline{\sigma_{b,U}}$ is onto $\sigma_b(U)$ which we henceforce denote $D_{b,U}$.


Look at $\overline{\inv_{b,U}} \colon X \to \RR^n$ defined in the statement of the theorem and recall that $\psi \in \BpHI_n(X)$ (Definition \ref{def:BHI}).
Consider some $c \in \kappa(U)$ and set $y=\overline{\sigma_{b,U}}(c)$.
Then $c=\kappa(u)$ for some $u \in U$ and by construction
\[
y=\overline{\sigma_{b,U}}(c) = \sigma_b(u) \in D_{b,U}.
\]
By the definition of $\sigma_b$,
\[
y=\fix(T_{b(u)})
\]
and by Definition \ref{defn:T_kappa phi}, $T_{b(u)}=T_{(c,\psi)}$ and
\[
y_i=\frac{c_i}{1+\psi(y)_i} \qquad (1 \leq i \leq n).
\]
Then $c_i=y_i(1+\psi(y)_i)$ and notice that $0 \leq y_i \leq 1$ so $y \in X$.
It follows that $c=\overline{\inv_{b,U}}(y)$.
We have shown that
\[
\overline{\inv_{b,U}}|_{D_{b,U}} \, \circ \, \overline{\sigma_{b,U}} \, =  \, \OP{id}_{\kappa(U)}.
\]
It follows that $\overline{\sigma_{b,U}}$ is injective, hence bijective onto its image $D_{b,U}$ and  $\overline{\inv_{b,U}}|_{D_{b,U}}$ is its inverse.

\ref{th:Fbar:reflect}.
Let $x \in X$.
Suppose that $\overline{\inv_{b,U}}(x) \in \kappa(U)$.
Then $\overline{\inv_{b,U}}(x)=\kappa(u)$ for some $u \in U$.
Thus, $\kappa(u)_i=x_i(1+\psi(x)_i)$ for all $i=1,\dots,n$.
Definition \ref{defn:T_kappa phi}
\[
x_i
=
\frac{\kappa(u)_i}{1+\psi(x)_i} 
= 
T_{(\kappa(u),\psi)}(x)_i.
\]
It follows that $T_{(\kappa(u),\psi)}(x)=x$.
Since $\psi=\varphi(u)$ and $b(u)=(\kappa(u),\varphi(u))$ is a scoring base, it follows that $\fix(T_{b(u)}(x)=x$ so $\sigma_b(u)=x$, see Definition \ref{def:weighted scoring base}.
Hence $x \in \sigma_b(U)=D_{b,U}$.

Conversely, suppose that $x \in D_{b,U}$.
By construction there exists $u \in U$ such that $x=\sigma_b(u)=\overline{\sigma_{b,U}}(\kappa(u))$.
By definition, $x=\fix(T_{\kappa(u),\varphi(u)})$, hence 
\[
x_i = 
T_{(\kappa(u),\varphi(u))}(x)_i =
T_{(\kappa(u),\psi)}(x)_i =
\frac{\kappa(u)_i}{1+\psi(x)_i}.
\]
It follows that 
\[
\kappa(u)_i = x_i(1+\psi(x)_i)=\overline{\inv_{b,U}}(x)_i \qquad (1 \leq i \leq n).
\] 
Thus, $\overline{\inv_{b,U}}(x)=\kappa(u) \in \kappa(U)$.

\ref{th:Fbar:homeomorphism}.
Apply Lemma \ref{lem:homeomorphism criterion} with the spaces $X=[0,1]^n=K$, $C=\kappa(U) \subseteq K$, $Y=\RR^n$ and $D=D_{b,U}$, and with $k=\overline{\inv_{b,U}}$ and $f=\overline{\sigma_{b,U}}$.
Condition \ref{lem:homeo:inverse} of that lemma is guaranteed by parts \ref{th:Fbar:exists} and \ref{th:Fbar:inverse} of this theorem which we have proven above, and condition  \ref{lem:homeo:reflecting} of Lemma \ref{lem:homeomorphism criterion} is part \ref{th:Fbar:reflect}. 
\end{proof}

Recall the concept of preservation of support and independence of supports from Definition \ref{def:preserve supports}.

\begin{proposition}\label{prop:bar inv preserves supports}
The function $\overline{\inv_{b,U}}$ defined in Theorem \ref{th:F bar factorisation} preserves supports.
\end{proposition}

\begin{proof}
$x_i=0 \iff x_i(1+\psi^I(x)_i)=0 \iff \overline{\inv_{b,U}}=0$.
\end{proof}

\begin{lemma}\label{lem:supports F kappa}
Let $(c,f)$ be a scoring base (Definition \ref{def:scoring base}).
Consider $T_{(c,f)} \colon X \to X$ (Definition \ref{defn:T_kappa phi} and Theorem \ref{theorem:T_c f in Base(X)}).
Then
\[
\fix(T_{c,f})_i = 0 \iff c_i=0.
\]
\end{lemma}

\begin{proof}
Set $x=\fix(T_{(c,f)})$.
By construction of $T_{(c,f)}$ we get $c_i=x_i(1+f(x)_i)$ and the result follows since $f \geq 0$ by hypothesis.
\end{proof}

\begin{theorem}\label{theorem:sigma_b preservation supports}
Let $b=(\kappa,\varphi)$ be a weighted scoring base.
Then $\sigma_b \colon W \to X$ preserves supports if and only if $\kappa \colon W \to K$ preserves supports.
\end{theorem}

\begin{proof}
Consider some $w \in W$.
By Lemma \ref{lem:supports F kappa} and the definition of $\sigma_b$ 
\[
\sigma_b(w)_i=0 
\iff
\fix(T_{(\kappa(u),\varphi(u)})_i=0
\iff
\kappa(w)_i=0.
\]
We have shown that $\supp(\sigma_b(w))=\supp(\kappa(w))$ for all $w \in W$.
So $\sigma_b$ preserves supports (i.e $\supp(\sigma_b(w))=\supp(w)$ for all $w$) if and only if $\kappa$ preserves supports (i.e $\supp(\kappa(w))=\supp(w)$ for all $w$).
\end{proof}

Let $\wp(n)$ be the power set of $\{1,\dots,n\}$.
The support gives a function
\[
\supp \colon \RR^n \to \wp(n)
\]
Given $I \in \wp(n)$ write 
\[
\RR^n(I) = \supp^{-1}(I) = \{ v \in \RR^n : \supp(v)=I\}.
\]
More generally, if $A \subseteq \RR^n$ we write $A(I):=A \cap \RR^n(I)$, the set of all $x \in A$ with support $I$.

With the terminology of Definition \ref{def:preserve supports}, the statement that a function 
$f \colon W \to Y$, where $Y$ is some set,  is independent of supports is equivalent to the assertion that $f$ is independent of weights on $W(I)$, see Definition \ref{def:independent of weights}, for all $I \in \wp(n)$.

\begin{proof}[\textbf{Proof of Theorem \ref{theorem:inverse reflection for based schemes}}]
Write $b=(\kappa,\varphi)$.
Clearly $\{W(I)\}_{I \in \wp(n)}$ forms a partition of $W$.
Since $\kappa$ preserves supports, it easily follows that $\kappa(W(I))=\kappa(W) \cap \RR^n(I)$ for every $I \in \wp(n)$.
In particular, $\{\kappa(W(I))\}_{I \in \wp(n)}$ forms a partition of $\kappa(W)$.

Set $D_b=\sigma_b(W)$.
By Theorem \ref{theorem:sigma_b preservation supports} $\sigma_b \colon W \to X$ preserves supports, hence $\sigma_{b}(W(I))=D_b(I)$.
Therefore $\{\sigma_b(W(I))\}_{I \in \wp(n)}$ forms a partition of $D_b$.

By Theorem \ref{th:F bar factorisation} for every $I \in \wp(n)$ there is an injective function
\[
\overline{\sigma_{b,W(I)}} \colon \kappa(W(I)) \to X
\]
with image $\sigma_b(W(I))$ such that $\overline{\sigma_{b,W(I)}} \circ \kappa|_{\kappa(W(I))} = \sigma_b|_{W(I)}$.
There is also a discerning right inverse for $\overline{\sigma_{b,W(I)}}$,
\[
\overline{\inv_{b,W(I)}} \colon X \to \RR^n.
\]
Since $\{\kappa(W(I))\}_I$ form a partition of $\kappa(W)$, we 
define $\overline{\sigma_b} \colon \kappa(W) \to X$ 
by the requirement
\[
\overline{\sigma_b}|_{\kappa(W(I))} = \overline{\sigma_{b,W(I)}}.
\]
Since the image of $\overline{\sigma_{b,W(I)}}$ is $D_b(I)$ and since $\{ D_b(I)\}_I$ forms a partition of $D_b$, it follows that $\overline{\sigma_b}$ is injective with image $D_b=\sigma_b(W)$.
For any $w \in W$ set $I=\supp(\kappa(w))$.
Since $\kappa$ preserves supports $\supp(w)=I$ and we get 
$\overline{\sigma_b} \circ \kappa(w) = \overline{\sigma_{b,W(I)}} \circ \kappa(w) = \sigma_b|_{W(I)}(w)$.
We deduce that
\[
\overline{\sigma_b} \circ \kappa = \sigma_b.
\]
Since $\{X(I)\}_{I \in \wp(n)}$ is clearly a partition of $X$, we define $\overline{\inv_b} \colon X \to \kappa(W)$ by the requirement
\[
\overline{\inv_b}|_{X(I)} = \overline{\inv_{b,W(I)}}|_{X(I)}.
\]
By the defining formula for $\overline{\inv_{b,W(I)}}|_{X(I)}$ in Theorem \ref{th:F bar factorisation}\ref{th:Fbar:inverse}, for any $x \in X$, if we set $I=\supp(x)$
\[
\overline{\inv_{b,W(I)}}(x)_i = x_i(1+\psi^I(x)_i).
\]
Thus, $\overline{\inv_b}$ we have just defined coincides with $\overline{\inv}$ in the statement of the theorem.
Notice that by this formula it is clear that $\overline{\inv_b}$ preserves supports since $\overline{\inv_b}_i=0 \iff x_i=0$.

Next we claim that $x \in \sigma_b(W) \iff \overline{\inv_b}(x) \in \kappa(W)$ and that in this case $\overline{\sigma_b} \circ \overline{\inv_b}(x)=x$.
Indeed, suppose that $x \in \sigma_b(W)$ and set $I=\supp(x)$.
Since $\sigma_b$ preserves supports, $x=\sigma_b(w)$ for some $w \in W(I)$.
Thus, $x \in \sigma_b(W(I))$.
Since $\overline{\sigma_{b,W(I)}}$ is a discerning right inverse for $\overline{\sigma_{b,{W(I)}}}$ it follows that $\overline{\inv_b}(x) = \overline{\inv_{b,W(I)}}(x) \in \kappa(W(I))$ as needed.
Furthermore, by the definition of $\overline{\sigma_b}$
\[
\overline{\sigma_b} (\overline{\inv_b}(x)) =
 \overline{\sigma_b} (\overline{\inv_{b,W(I)}}(x)) =
 \overline{\sigma_{b,W(I)}} (\overline{\inv_{b,W(I)}}(x)) = x.
\]
Conversely, suppose that $\overline{\inv_b}(x) \in \kappa(W)$.
Set $I=\supp(x)$.
Since $\overline{\inv_b}$ preserves supports, $\overline{\inv_b}(x) \in \kappa(W) \cap \RR^n(I) = \kappa(W(I))$.
By definition, $\overline{\inv_b}(x)=\overline{\inv_{b,W(I)}}(x)$ so $\overline{\inv_{b,W(I)}}(x) \in \kappa(W(I))$.
Since $\overline{\inv_{b,W(I)}}$ is a discerning right inverse for $\overline{\sigma_{b,W(I)}}$, it follows that $x \in \sigma_b(W(I)) \subseteq \sigma_b(W)$.
We have thus shown that $\overline{\inv_b}$ is a discerning right inverse for $\overline{\sigma_b}$.

Define $\inv_b \colon X \to \RR^n$ by $\inv_b = \tilde{\kappa}^{-1} \circ \overline{\inv_b}$.
Then $\inv_b$ coincides with $\inv$ in the statement of the theorem.
Since $\sigma_b=\overline{\sigma_b} \circ \kappa$ we get
\[
x \in \sigma_b(W) 
\iff
x \in \overline{\sigma_b}(\kappa(W))
\iff
\overline{\inv_b}(x) \in \sigma_b(W).
\]
Furthermore, since $\overline{\inv_b}$ is a discerning right inverse of $\overline{\sigma_b}$
\[
\sigma_b \circ \inv_b|_{D_b}
=
(\overline{\sigma_b} \circ \kappa) \circ (\tilde{\kappa}^{-1} \circ \overline{\inv_b}|_{D_b})
=
\overline{\sigma_b} \circ \overline{\inv_b}|_{D_b}
=
\OP{id}_{D_b}.
\]
Thus, we have shown that  $\inv_b$ is a discerning right inverse for $\sigma_b$ as needed.

Finally, suppose that $\varphi$ is independent of weights with values $\psi \colon X \to [0,\infty)^n$ which is continuous.
Since $\tilde{\kappa}$ is a homeomorphism and $W$ is compact, $\kappa(W)=\tilde{\kappa}(W)$ is compact.
Theorem \ref{th:F bar factorisation}\ref{th:Fbar:homeomorphism} shows that $\sigma_b=\overline{\sigma_b} \circ \kappa$ is a homemorphism onto its image.
\end{proof}

\begin{proof}[\textbf{Proof of Theorem \ref{theorem:abstract projective preference}}]
Set $I = \supp(y)$.
By the hypothesis $\varphi$ is independent of weights on $W(I)$.
Let $\psi \colon X \to [0,\infty)^n$ be the $W(I)$-value of $\varphi$.

Let $\overline{\sigma_{b,W(I)}} \colon \kappa(W(I)) \to X$  and $\overline{\inv_{b,W(I)}} \colon X \to \RR^n$ be as in Theorem \ref{th:F bar factorisation}.
Thus, $\overline{\sigma_{b,W(I)}}$ is bijective onto $\sigma_b(W(I))$.
By Theorem \ref{theorem:sigma_b preservation supports} $\sigma_b$ preserves supports, and it easily follows that $\sigma_b(W(I))=D_b(I)$ where $D_b=\sigma_b(W)$.
Similarly, since $\kappa$ preserves supports, 
\[
\kappa(W(I))=\RR^n(I) \cap \kappa(W).
\]
Clearly, $t \cdot y \in X(I)$ for any $0 < t \leq 1$.
By definition of $\overline{\inv_{b,W(I)}}$ in Theorem \ref{th:F bar factorisation} and since by hypothesis $\psi_i$ are homogeneous
\[
\overline{\inv_{b,W(I)}}(t \cdot y)_i = ty_i(1+\psi(ty)_i) =ty_i(1+t\cdot \psi(y)_i) \xto{ \  t \to 0  \ } 0.
\]
It follows that there exists $s>0$ such that $\overline{\inv_{b,W(I)}}(ty) \in U$ for every $0 <t \leq s$.
Moreover, for any $t>0$ it is clear that $\overline{\inv_{b,W(I)}}(ty)_i=0$ if and only if $y_i=0$, hence $\supp(\overline{\inv_{b,W(I)}}(ty))=\supp(y)=I$ for any $t>0$.
Since $\kappa$ preserves supports, for any $0<t \leq s$.
\[
\overline{\inv_{b,W(I)}}(t \cdot y) \in \RR^n(I) \cap U \subseteq \RR^n(I) \cap \kappa(W) = \kappa(W(I)).
\]
Part \ref{th:Fbar:reflect} of Theorem \ref{th:F bar factorisation} implies that $ty \in \sigma_b(W(\alpha)) \subseteq D_{(\kappa,\varphi)}$.
\end{proof}

\begin{proof}[\textbf{Proof of Theorem \ref{theorem:abstract cone solution}}]
Consider some $x \in D_b=\sigma_b(W)$.
We need to prove that the interval  $[0,x] \subseteq D_b$, i.e $tx \in D_b$ for any $0 \leq t \leq 1$.
By Theorem \ref{theorem:sigma_b preservation supports} $\sigma_b$ preserves supports.
In particular, $\sigma_b(0)=0$ (where $0 \in W$ and $0 \in X$).
Thus, for the rest of the proof we assume that $0 < t \leq 1$.

Set $I=\supp(x)$.
Since $x=\sigma_b(w)$ for some $w \in W$ and Since $\sigma_b$ preserves supports, it follows that $w \in W(I)$.
By hypothesis $\varphi$ is independent of supports on $W(I)$.
Let $\psi \colon X \to [0,\infty)^n$ denote the $W(I)$-value of $\varphi$.
Then $\psi \in \BpHI_n(X)$.

Consider $\overline{\sigma_b,W(I)} \colon \kappa(W(I)) \to X$ and $\overline{\inv_{b,W(I)}} \colon X \to \RR^n$ from Theorem \ref{th:F bar factorisation}.
Recall that the image of $\overline{\sigma_b}$ is $\sigma_b(W(I))$.
Since $x=\sigma_b(w)=\overline{\sigma_{b,W(I)}}(\kappa(w))$, it follows from Theorem \ref{th:F bar factorisation} that $\overline{\inv_{b,W(I)}}(x) =\kappa(w)$.
By the construction of $\overline{\inv_{b,W(I)}}$ and since $\psi$ is homogeneous
\begin{align*}
\overline{\inv_{b,W(I)}}(tx)_i &=
tx_i(1+\psi(tx)_i)
\\
&= tx_i (1+t \psi(x)_i) 
\\
& \leq x_i(1+\psi(x)_i)
\\
&= \overline{\inv_{b,W(I)}}(x)_i
\\
&= \kappa(x)_i.
\end{align*}
Thus, $\overline{\inv_{b,W(I)}}(tx) \preceq \kappa(w) \in \kappa(W(I))$.
It is clear from the defining formula of $\overline{\inv_{b,W(I)}}$ that it preserves supports, so $\supp(\overline{\inv_{b,W(I)}}(tx))=\supp(tx)=\supp(x)=I$.
Since $\kappa(W)$ is $\supp\OP{-}\preceq$ closed in $K$ and since $\kappa(W(I))=\kappa(W) \cap \RR^n(I)$, we deduce that $\overline{\inv_{b,W(I)}}(tx) \in \kappa(W(I))$.
By Theorem \ref{th:F bar factorisation}\ref{th:Fbar:reflect} implies that $tx \in \sigma_b(W(I)) \subseteq D_b$.
\end{proof}

Since $\lambda$ and $\mu$ are independent of weights, given $I \in \wp(n)$ we write $\lambda^I, \mu^I \in \RR^n$ for the value of these functions on $\RR^n(I)$.

\begin{proof}[\textbf{Proof of Proposition \ref{prop:kappa phi of Lp}}]
Since $0 \leq \mu(w)_i \leq 1$ it is clear that $\kappa(w) \in K$ for all $w \in W$ so $\kappa \colon W \to K$ is well defined.

Choose some $w \in W$ and set $\psi=\varphi(w)$.
We need to show that $\psi \in \BpHI_n(X)$, see Definition \ref{def:BHI}.
It is clear that $\psi \geq 0$ since $\| \ \|_p \geq 0$.
It is also clear that $\psi$ is continuous since $\| \ \|_p$ is continuous as well as $x \mapsto v \wedge x$ for a fixed $v \in \RR^n$.
Then $\psi$ is bounded since 
Since $X$ is compact, $\psi$ is bounded.
For any $t \geq 0$ it is clear that
\[
\psi(tx)_i 
=
\| \lambda(w)(a_{i*}) \wedge tx\|_p 
= t \| \lambda(w)(a_{i*}) \wedge x\|_p 
=
t \cdot \psi(x)_i.
\]
So $\psi$ is homogeneous.
Suppose that $x \preceq x'$ in $X$, i.e $x_i \leq x'_i$.
Since $\lambda(w)(a_{i,*}) \geq 0$ it follows that $0 \leq \lambda(w)(a_{i,*})_jx_j \leq \lambda(w)(a_{i,*})_j x_j$.
By the definition of the $L^p$-norms \eqref{eqn:Lp norm} it is easy to check that $\|\lambda(w)(a_{i,*})(a_{i,*}) \wedge x\|_p \leq \|\lambda(w)(a_{i,*}) \wedge x'\|_p$, so $\psi(x) \preceq \psi(x')$.
Hence, $\psi$ is increasing.
\end{proof}

\begin{proposition}\label{prop:Lp based kappa phi are good}
Let $(\kappa,\varphi)$ be the weighted scoring base defined in \eqref{eqn:Lp lambda mu A kappa phi}.
Then
\begin{enumerate}[label=(\roman*)]

\item\label{prop:Lp good:kappa support}
$\kappa$ is the restriction of a bijective support-preserving function $\tilde{\kappa} \colon \RR^n \to \RR^n$.
If $\mu$ is independent of weights then $\tilde{\kappa}$ is a homeomorphism.

\item
\label{prop:Lp good:kappa nbhd}
$\kappa(W)$ contains a neighbourhood of $0$ in $K$.

\item
\label{prop:Lp good:closed}
$\kappa(W)$ is $(\supp,\preceq)$-closed in $K$.

\item
\label{prop:Lp good:phi indep weights}
$\varphi$ is independent of supports.

\item
\label{prop:Lp good:phi cts}
$\varphi(w) \colon X \to [0,\infty)^n$ is continuous for all $w \in W$.
\end{enumerate}
\end{proposition}

\begin{proof}
Define $\tilde{\kappa} \colon \RR^n \to \RR^n$ by
\[
\tilde{\kappa}(u)_i = \mu(u)_i u_i.
\]
Then $\tilde{\kappa}$ preserves supports because $\mu(u)_i \neq 0$ so $\tilde{\kappa}(u)_i= \iff u_i=0$.
If $\mu$ is independent of weights, i.e constant with value $d \in (0,1)^n$, then $\tilde{\kappa}$ is the linear transformation given by the diagonal matrix $D=\diag(d_1,\dots,d_n)$.
In particular, it is a homeomorphism.
This proves \ref{prop:Lp good:kappa support}.

Let $\mu^I \in (0,1)^n$ denote the constant value of $\mu$ on $W(I)$.
Then $\tilde{\kappa}|_{\RR^n(I)}$ is multiplication by the matrix $D^I=\diag(\mu^I_1,\dots,\mu^I_n)$ which preserves supports.
It follows that 
\[
\kappa(W) = \coprod_{I \in \wp(n)} \left(\prod_{j=1}^n [0,\mu^I_j]\right).
\]
It is now clear that $\kappa(W)$ is $(\supp,\preceq)$-closed and that it contains the cube $[0,m]$ where $m=\min\{\mu^I_j\}$.
This proves \ref{prop:Lp good:kappa nbhd} and \ref{prop:Lp good:closed}.

Since $\lambda$ is independent of supports it is clear that $\varphi$ is independent of supports with $\varphi^{\RR^n(I)}(x)_i = \| \lambda^{\RR^n(I)}(a_{i,*}) \wedge x\|_p$.
It also follows that $\varphi^{\RR^n(I)}$ is continuous.
Thus, \ref{prop:Lp good:phi indep weights} and \ref{prop:Lp good:phi cts} follow.
\end{proof}

\begin{proof}[\textbf{Proof of Theorem \ref{theorem:Lp-based solve everything}}]
By Proposition \ref{prop:Lp based kappa phi are good}\ref{prop:Lp good:kappa support} and \ref{prop:Lp good:phi indep weights}, all the conditions of Theorem \ref{theorem:inverse reflection for based schemes} are satisfied.
We deduce that $\sigma_{\kappa,\varphi}$ admits a solution to the inverse and reflection problems \ref{problem:inverse}, \ref{problem:reflection}.
By that theorem, the inverse function is computed as follows.
Define $\overline{\inv} \colon X \to \RR^n$ by
\[
\overline{\inv}(x)_i = x_i(1+\varphi^{\supp(x)}(x)_i) = x_i(1+ \|\lambda(x)(a_{i,*}) \wedge x\|_p).
\]
Set $I=\supp(x)$.
Since $\tilde{\kappa}$ is given on ${\RR^n(I)}$ by the  matrix $\diag(\mu^{I}_1,\dots,\mu^{I}_n)$, and since $\mu^I=\mu(x)$ and $\lambda^I=\lambda(x)$,
\[
\inv(x)_i
=
\tilde{\kappa}^{-1}(\overline{\inv(x)})_i
=
\tfrac{1}{\mu(x)_i} x_i(1+ \|\lambda(x)(a_{i,*}) \wedge x\|_p).
\]
If $\lambda,\mu$ are independent of weights, i.e they are constant, then by Proposition \ref{prop:Lp based kappa phi are good}\ref{prop:Lp good:kappa support} and \ref{prop:Lp good:phi cts} the conditions for a solution of the topological reflections problem hold.

Proposition \ref{prop:Lp based kappa phi are good}\ref{prop:Lp good:kappa support},\ref{prop:Lp good:kappa nbhd},\ref{prop:Lp good:closed} and \ref{prop:Lp good:phi indep weights} show that all the conditions of Theorems \ref{theorem:abstract projective preference} and \ref{theorem:abstract cone solution} are satisfied so $\sigma_{\kappa,\varphi}$ solve the projective preference ordering problem \ref{problem:preference} and the radiality problem \ref{problem:radiality}.
\end{proof}

\begin{proposition}\label{prop:monotonicity => inclusion}
Let $b=(\kappa,\varphi)$ and $b'=(\kappa',\varphi')$ be weighted scoring bases.
Suppose that
\begin{enumerate}[label=(\roman*)]
\item\label{prop:monotonicity => inclusion:supports}
$\kappa$ and $\kappa'$ preserve supports.

\item\label{prop:monotonicity => inclusion:closure}
$\kappa(W)$ and $\kappa'(W)$ are $(\supp,\preceq)$-closed in $K$.

\item\label{prop:monotonicity => inclusion:weights}
$\varphi$ and $\varphi'$ are independent of supports.
Let $\psi^I$ and $\psi'{}^I$ denote their $W(I)$-value where $I \in \wp(n)$.

\item\label{prop:monotonicity => inclusion:lambda}
For every $I \in \wp(n)$ there exists $0<\alpha^I \leq 1$ such that that
\begin{itemize}
\item
$\psi'{}^I \leq \alpha^I \cdot \psi^I$

\item
$\kappa|_{W(I)} \leq \alpha^I \cdot \kappa'|_{W(I)}$. 
\end{itemize}
\end{enumerate}
Then $D_{(\kappa,\varphi)} \subseteq D_{(\kappa',\varphi')}$.
\end{proposition}

\begin{proof}
Consider some $x \in D_{(\kappa,\varphi)}$ and set $I=\supp(x)$.
Then $x=\sigma_b(w)$ for some $w \in W$.
By Theorem \ref{theorem:sigma_b preservation supports} $\sigma_b$ preserve supports.
Therefore, $w \in W(I)$ and $x \in \sigma_b(W(I)) = \sigma_b(W) \cap \RR^n(I)$.
In particular $\kappa(W(I)) \subseteq K(I)$.

Apply Theorem \ref{th:F bar factorisation} to obtain functions
\begin{align*}
& \overline{\sigma_{b,W(I)}} \colon \kappa(W(I)) \to X \\
& \overline{\inv_{b,W(I)}} \colon X \to \RR^n
\end{align*}
where $\overline{\sigma_{b,W(I)}}$ is bijective onto $\sigma_b(W(I))$ and $\overline{\inv_{b,W(I)}}$ is a discerning right inverse.
Then 
\[
\overline{\inv_{b,W(I)}}(x) = \overline{\inv_{b,W(I)}}(\sigma_b(w)) = \overline{\inv_{b,W(I)}}(\overline{\sigma_{b,W(I)}}\kappa(w)) = \kappa(w) \in \kappa(W(I)).
\]
By Proposition \ref{prop:bar inv preserves supports} $\overline{\inv_{b,W(I)}}$ preserves supports.
Since $0 < \alpha^I \leq 1$ we get $\alpha^I x \in X(I)$, so 
\[
\overline{\inv_{b,W(I)}}(\alpha^I x) \in \RR^n(I).
\]
Clearly $\alpha^I x \preceq x$ and since $\psi^I$ is increasing
\begin{multline*}
0 
\leq 
\overline{\inv_{b,W(I)}}(\alpha^Ix) 
= 
\alpha^I x_i(1+\psi^I(\alpha^I x)_i)
\leq
\\
x_i(1+\psi^I(x)_i)
=
\overline{\inv_{b,W(I)}}(x)_i
=\kappa(w)_i \leq 1.
\end{multline*}
It follows that $\overline{\inv_{b,W(I)}}(\alpha^I x) \in K \cap \RR^n(I)=K(I)$ and $\overline{\inv_{b,W(I)}}(\alpha^I x) \preceq \kappa(w) \in \kappa(W(I))$.

By assumption $\kappa(W)$ is $(\supp,\preceq)$-closed in $K$, so $\overline{\inv_{b,W(I)}}(\alpha^I x) \in \kappa(W) \cap \RR^n(I) =\kappa(W(I))$.
Therefore $\overline{\inv_{b,W(I)}}(\alpha^I x)=\kappa(U)$ for some $u \in W(I)$.
It follows that
\begin{align*}
\overline{\inv_{b',W(I)}}(x)_i &=
x_i(1+\psi'{}^I(x)_i)
\\
& \leq 
x_i(1+\alpha^I \psi^I(x)_i)
\\
& = 
\tfrac{1}{\alpha^I}(\alpha^I x_i)(1+\psi^I(\alpha^I x)_i)
\\
&=
\tfrac{1}{\alpha^I} \overline{\inv_{b,W(I)}}(\alpha^I x)_i 
\\
& =
\tfrac{1}{\alpha^I} \kappa(u)
\\
& \leq 
\tfrac{1}{\alpha^I} \alpha^I \kappa'(u)_i
\\
&=
\kappa'(u)_i.
\end{align*}
Thus, we have shown that
\[
\overline{\inv_{b',W(I)}}(x) \preceq \kappa'(u).
\]
By Proposition \ref{prop:bar inv preserves supports} $\supp(\overline{\inv_{b',W(I)}}(x))=\supp(x)=I$ and also $\supp(\kappa'(u))=\supp(u)=I$ and $0 \leq \overline{\inv_{b',W(I)}}(x)_i \leq \kappa'(u)_i \leq 1$.
Since $\kappa'(W)$ is $(\supp,\preceq)$-closed in $K$ we get $\overline{\inv_{b,W(I)}}(x) \in \kappa'(W) \cap \RR^n(I)=\kappa'(W(I))$.
By Theorem \ref{th:F bar factorisation}\ref{th:Fbar:reflect} $x \in \sigma_{b'}(W(I)) \subseteq \sigma_{b'}(W)=D_{(\kappa',\varphi')}$.
\end{proof}

\begin{proof}[\textbf{Proof of Theorem \ref{theorem:Lp subseteq Lq}}]
Let $b=(\kappa,\varphi)$ be the $(L^p,\lambda,\mu,A)$ weighted scoring base and $b'=(\kappa',\varphi')$ be the $(L^q,\lambda,\mu,A)$  weighted scoring base.
By construction $\kappa=\kappa'$.

Conditions \ref{prop:monotonicity => inclusion:supports}, \ref{prop:monotonicity => inclusion:closure} and \ref{prop:monotonicity => inclusion:weights} of Proposition \ref{prop:monotonicity => inclusion} hold by Proposition \ref{prop:Lp based kappa phi are good}.
Condition \ref{prop:monotonicity => inclusion:lambda} also holds by setting $\alpha^I=1$ since $\| \ \|_q \leq \| \ \|_p$ and $\kappa \leq \kappa'$ (in fact, equality).
\end{proof}

\begin{proof}[\textbf{Proof of Proposition \ref{prop:HC MB CB are Lp rough version}}]
Let $\bar{1} \in \RR^n$ denote the vector $(1,\dots,1)$. We first highlight the following three items (and their proofs).

\begin{enumerate}
    \item \label{prop:HC is Lp}
The gradual semantics $\Sigma_{\HC}$ is an abstract $(L^p,\lambda,\mu)$-based gradual semantics with $p=1$ and $\mu \colon \RR^n \to (0,1]^n$ and $\lambda \colon \RR^n \to \End([0,\infty)^n)$ given by
\begin{align*}
& \mu(v) = \bar{1} \qquad (v \in \RR^n)
\\
& \lambda(v) = \OP{id} \in \End([0,\infty)^n) \qquad (v \in \RR^n).
\end{align*}

\begin{proof}
By construction $\tilde{\kappa} \colon \RR^n \to \RR^n$ is the identity because $\mu(v)_i=1$ for all $v \in \RR^n$.
Hence $\kappa \colon W \to K$ is the identity of $[0,1]^n$.
The function $\varphi \colon W \to \Func(X,[0,\infty)^n)$ is 
\[
\varphi(w)(x)_i = \| a_{i,*} \wedge x\|_1 = \sum_j a_{i,j} x_j.
\]
By definition, $\sigma_{(L^p,\lambda,\mu,A)}(w)=\fix(T_{(\kappa(w),\varphi(w))})$ and it is the limit of any $T_{(\kappa(w),\varphi(w))}$-sequence in $X$ (Definition \ref{def:scoring dynamic} and Theorem \ref{theorem:T_c f in Base(X)}).
Now,
\[
T_{(\kappa(w),\varphi(w))}(x)_i = \frac{\kappa(w)_i}{1+\varphi(w)(x)_i} = \frac{w_i}{1+\| a_{i,*} \wedge x\|_1}.
\]
Since $A=(a_{i,j})$ is the adjacency matrix of a graph $\G$,
\[
\| a_{i,*} \wedge x\|_1 = \sum_j a_{i,j}x_j = \sum_{j \in \Att(i)} x_j.
\]
From the construction of $\Sigma_{\HC}$, $\Sigma_{\HC}^{A,w} = \fix(T_{(\kappa(w),\varphi(w))})=\sigma_{(L^p,\lambda,\mu,A)}(w)$.
Thus, $\Sigma_{\HC} = \Sigma_{(L^p,\lambda,\mu)}$.
\end{proof}

\item \label{prop:MB is Lp}
The gradual semantics $\Sigma_{\MB}$ is an abstract $(L^p,\lambda,\mu)$-based gradual semantics with $p=\infty$ and $\mu \colon \RR^n \to (0,1]^n$ and $\lambda \colon \RR^n \to \End([0,\infty)^n)$ given by
\begin{align*}
& \mu(v) = \bar{1} \qquad (v \in \RR^n)
\\
& \lambda(v) = \OP{id} \in \End([0,\infty)^n) \qquad (v \in \RR^n).
\end{align*}

\begin{proof}
The proof of item \ref{prop:HC is Lp} can be read verbatim upon replacing $p=1$ with $p=\infty$ and $\Sigma_{\HC}$ with $\Sigma_{\MB}$ in Example \ref{example:MB semantics} and observing that 
\[
\| a_{i,*} \wedge x\|_\infty = \max_j \{a_{i,j}x_j\} = \max_{j \in \Att(i)} x_j.
\]
The details are left to the reader.
\end{proof}

\item \label{prop:CB is Lp}
The gradual semantics $\Sigma_{\CB}$ is an abstract $(L^p,\lambda,\mu)$-based gradual semantics with $p=1$ and $\mu \colon \RR^n \to (0,1]^n$ and $\lambda \colon \RR^n \to \End([0,\infty)^n)$ given as follows.
First, define $\theta \colon \RR^n \to \RR^n$ by
\[
\theta(v)_i = \| a_{i,*} \wedge \supp(v)\|_1
\]
Then define $\mu$ and $\lambda$ by
\begin{align*}
& \mu(v)_i = \frac{1}{1+\theta(v)_i}
\\
& \lambda(v)_i = 
\left\{
\begin{array}{ll}
\tfrac{\mu(v)_i}{\theta(v)_i} \cdot (a_{i,*} \wedge \supp(v)) & \text{if $\theta(v)_i \neq 0$}
\\
0  & \text{if $\theta(v)_i = 0$}
\end{array}
\right.
\end{align*}

\begin{proof}
Choose some $w \in W$.
By the definition of $\kappa$ and $\varphi$ in Subsection \ref{subsec:Lp lambda mu A}
\begin{align*}
T_{(\kappa(w),\varphi(w))}(x)_i 
&= 
\frac{w_i/(1+\theta(w)_i)}{1+ \tfrac{\mu(v)_i}{\theta(v)_i} \| \supp(w) \wedge a_{i,*} \wedge x\|_1} 
\\
&=
\frac{w_i}{1+\theta(w)_i+ \tfrac{1}{\theta(w)_i} \| \supp(w) \wedge a_{i,*} \wedge x\|_1}
\end{align*}
with the convention that $\tfrac{1}{\theta(w)_i}=0$ if $\theta(w)_i=0$.
Now, $\sigma_{(L^p,\lambda,\mu,A)}(w) = \fix(T_{(\kappa(w),\varphi(w))})$ is the limit of any $T_{(\kappa(w),\varphi(w))}$-sequence in $X$.
We observe that since $A=(a_{i,j})$ is the adjacency matrix of $\G$,
\begin{align*}
& \theta(w)_i = | \Att^*(i)|
\\
& \| \supp(w) \wedge a_{i,*} \wedge x\|_1 = \sum_{j \in \Att^*(i)} a_{i,j} x_j = \sum_{j \in \Att^*(i)} x_j.
\end{align*}
It follows from the construction of $\Sigma_{\CB}$ in Example \ref{example:CB semantics} that $\Sigma_{\CB}^{\G,w}=\sigma_{(L^p,\lambda,\mu,A)}$.
Thus, $\Sigma_{\CB}=\Sigma_{(L^p,\lambda,\mu)}$.
\end{proof}

\end{enumerate}

We can conclude from items \ref{prop:HC is Lp}, \ref{prop:MB is Lp} and \ref{prop:CB is Lp} and Theorem \ref{theorem:Lp-based solve everything}.
\end{proof}

\end{document}